\documentclass[letterpaper]{article} 
\usepackage{aaai25}  
\usepackage{times}  
\usepackage{helvet}  
\usepackage{courier}  
\usepackage[hyphens]{url}  
\usepackage{graphicx} 
\usepackage{natbib}  
\usepackage{caption} 
\usepackage{algorithm}
\usepackage{algorithmic}

\usepackage{newfloat}
\usepackage{listings}
\DeclareCaptionStyle{ruled}{labelfont=normalfont,labelsep=colon,strut=off} 
\floatstyle{ruled}
\newfloat{listing}{tb}{lst}{}
\floatname{listing}{Listing}
\usepackage{booktabs}

\usepackage{amssymb, amsmath, amsthm}

\usepackage[polish,english]{babel}
\usepackage[T1,nomathsymbols]{polski}
\usepackage[utf8]{inputenc}

\usepackage{cleveref} 
\usepackage{xcolor}
\usepackage{subcaption}
\usepackage{MnSymbol}
\usepackage{thmtools}
\usepackage{multirow}

\usepackage{tikz}
\usetikzlibrary{decorations.pathreplacing}
\newcommand{\rulesep}{\unskip\hfill{\color{black}\vrule}\hfill\ignorespaces}
\usepackage{pgfplots}
\pgfplotsset{compat=1.18}
\usetikzlibrary{arrows,backgrounds,automata,topaths,patterns,decorations.pathreplacing,decorations.pathmorphing}
\usetikzlibrary{arrows, chains, positioning, quotes, shapes.geometric}
\usetikzlibrary{calc}
\usetikzlibrary{fit}
\usetikzlibrary{arrows}
\usetikzlibrary{positioning,automata}

\usepackage{thm-restate}
\declaretheorem[name=Theorem]{thm}

\newtheorem{definition}[thm]{Definition}
\newtheorem{lemma}[thm]{Lemma}

\newtheorem*{claim*}{Claim}

\definecolor{mygreen}{HTML}{9BF3AA}
\definecolor{myred}{HTML}{F39B9B}
\definecolor{myblue}{HTML}{9BB0F3}

  \definecolor{SoftBlue}{HTML}{89CFF0}
  \definecolor{MintGreen}{HTML}{98FF98}
  \definecolor{LightCoral}{HTML}{F08080}
  \definecolor{SoftPurple}{HTML}{D8BFD8}
  \definecolor{PaleYellow}{HTML}{FFFFE0}
  \definecolor{SkyBlue}{HTML}{87CEEB}
  \definecolor{Peach}{HTML}{FFDAB9}
  \definecolor{Lavender}{HTML}{E6E6FA}
  \definecolor{LightPink}{HTML}{FFB6C1}
  \definecolor{Cream}{HTML}{FFFDD0}
  \definecolor{Aqua}{HTML}{00FFFF}
  \definecolor{Chartreuse}{HTML}{7FFF00}
  \definecolor{Crimson}{HTML}{DC143C}
  \definecolor{DarkOrange}{HTML}{FF8C00}
  \definecolor{Goldenrod}{HTML}{DAA520}
  \definecolor{HotPink}{HTML}{FF69B4}
  \definecolor{IndianRed}{HTML}{CD5C5C}
  \definecolor{LightSeaGreen}{HTML}{20B2AA}
  \definecolor{MediumPurple}{HTML}{9370DB}
  \definecolor{Salmon}{HTML}{FA8072}

\definecolor{LightGray}{RGB}{150,150,150}
\definecolor{Green1}{RGB}{171,255,177}
\definecolor{Green2}{RGB}{65,243,77}
\definecolor{Green3}{RGB}{120,165,101}
\definecolor{Brown}{RGB}{243,190,130}
\definecolor{Blue1}{RGB}{93,114,164}
\definecolor{Blue2}{RGB}{213,223,245}
\definecolor{Blue3}{RGB}{42,69,176}
\definecolor{Blue4}{RGB}{240,246,255}
\definecolor{Blue5}{RGB}{50,100,200}
\definecolor{Red1}{RGB}{255,89,100}
\definecolor{Light}{RGB}{102,205,170}

\usepackage[disable]{todonotes}

\newcommand{\pwline}[1]{\todo[author=Przemek,inline,backgroundcolor=green!20]{#1}}

\newcommand{\mrline}[1]{\todo[author=Michael,inline,backgroundcolor=red!20]{#1}}

\newcommand{\N}{\ensuremath{\mathbb{N}}}

\newcommand*{\ldblbrace}{\{\mskip-5mu\{}
\newcommand*{\rdblbrace}{\}\mskip-5mu\}}

\newcommand{\W}{\ensuremath{\mathbf{W}}}
\newcommand{\con}{\ensuremath{\,||\,}}

\newcommand{\G}{\ensuremath{G}}
\newcommand{\KG}{\ensuremath{KG}}
\newcommand{\NH}{\ensuremath{\mathcal{N}}}
\newcommand{\NHg}{\mathcal{N}_{\mathsf{glob},r}}
\newcommand{\NHl}{\mathcal{N}_{\mathsf{loc},r}}
\newcommand{\Ag}{A_{\mathsf{glob}}}
\newcommand{\Al}{A_{\mathsf{loc}}}
\newcommand{\Bg}{B_{\mathsf{glob}}}
\newcommand{\Bl}{B_{\mathsf{loc}}}

\newcommand{\x}{\ensuremath{\mathbf{x}}}

\newcommand{\h}{\ensuremath{\mathbf{h}}}
\newcommand{\hh}{\ensuremath{\overline{\mathbf{h}}}}
\newcommand{\bb}{\ensuremath{\mathbf{b}}}

\newcommand{\com}{\ensuremath{\mathsf{COM}}}
\newcommand{\agg}{\ensuremath{\mathsf{AGG}}}

\newcommand{\relu}{\ensuremath{\mathsf{ReLu}}}

\newcommand{\rwl}{\ensuremath{\mathsf{rwl}}}
\newcommand{\rwlg}{\mathsf{rwl}_{\mathsf{glob}}}
\newcommand{\rwll}{\mathsf{rwl}_{\mathsf{loc}}}
\newcommand{\sign}{\ensuremath{\mathsf{sign}}}
\newcommand{\aggsum}{\ensuremath{\mathsf{SUM}}}

\newcommand{\TG}{\ensuremath{TG}}
\newcommand{\TV}{\ensuremath{t\text{-}nodes}}
\newcommand{\D}{\ensuremath{D}}
\newcommand{\T}{\ensuremath{\mathsf{time}}}

\newcommand{\fone}
{\ensuremath{\bigstar}}
\newcommand{\Kone}[1]{\mathcal{K}_{\mathsf{glob}}(#1)}
\newcommand{\Ktwo}[1]{\mathcal{K}_{\mathsf{loc}}(#1)}

\newcommand{\GNN}{\ensuremath{\mathsf{GNN}}}

\newcommand{\model}{\ensuremath{\mathcal{A}}}
\newcommand{\modelB}{\ensuremath{\mathcal{B}}}
\newcommand{\MPGNN}{\ensuremath{\mathsf{MP}\text{-}\mathsf{GNN}}}

\newcommand{\TGNN}{\ensuremath{\mathsf{TGNN}}}

\newcommand{\MPTGNN}{\ensuremath{\mathsf{MP}\text{-}\mathsf{TGNN}}}

\newcommand{\glob}{global}
\newcommand{\loc}{local}

\title{Expressive Power of Temporal Message Passing}
\author{
Przemysław Andrzej Wałęga\textsuperscript{\rm 1}, 
Michael Rawson\textsuperscript{\rm 2}
}
\affiliations{
    \textsuperscript{\rm 1}University of Oxford, Queen Mary University of London, UK\\
    \textsuperscript{\rm 2}TU Wien, Austria
    
}

\usepackage{bibentry}

\begin{document}

\maketitle


\begin{abstract}
Graph neural networks (GNNs) have recently been adapted to temporal settings, often employing temporal versions of the message-passing mechanism known from GNNs. We divide temporal message passing mechanisms from literature into two main types: global and local, and establish Weisfeiler-Leman characterisations for both. This allows us to formally analyse expressive power of temporal message-passing models. We show that global and local temporal message-passing mechanisms have  incomparable expressive power when applied to arbitrary temporal graphs. However, the local mechanism is strictly more expressive than the global mechanism when applied to colour-persistent temporal graphs, whose node colours are initially the same in all time points. Our theoretical findings are supported by experimental evidence, underlining practical implications of our analysis.
\end{abstract}

\pagestyle{plain}

\section{Introduction}

\newsavebox\Tempgraph
\sbox\Tempgraph{\begin{tikzpicture}[
dot/.style = {draw, circle, minimum size=#1,
              inner sep=0pt, outer sep=0pt},
dot/.default = 6pt
]
\scriptsize
\pgfmathsetmacro{\tline}{-0.25}
\pgfmathsetmacro{\w}{1.3}
\pgfmathsetmacro{\h}{1.5}
\pgfmathsetmacro{\inh}{0.7}
\pgfmathsetmacro{\dist}{1.5}
\pgfmathsetmacro{\Ax}{0.5}
\pgfmathsetmacro{\Ay}{1.4}
\pgfmathsetmacro{\Bx}{0.4}
\pgfmathsetmacro{\By}{0.5}
\pgfmathsetmacro{\Cx}{1}
\pgfmathsetmacro{\Cy}{1}


\foreach \x in {1,...,4}
{
\draw[fill=gray!90!black,opacity=0.2] (\x*\dist,0) -- (\x*\dist+\w,\inh) -- (\x*\dist+\w,\inh+\h) -- (\x*\dist,\h) -- cycle;


\node (A\x) at (\x*\dist+\Ax, 0+\Ay) {};
\node (B\x) at (\x*\dist+\Bx, 0+\By) {};
\node (C\x) at (\x*\dist+\Cx, 0+\Cy) {};
}

\draw[thick] (A2) -- (B2);
\draw[thick](B3) -- (C3);
\draw[thick] (A4) -- (C4);
\draw[thick] (B4) -- (C4);

\node[dot=9pt,draw=black,fill=myblue] at (A1) {$a$};
\node[dot=9pt,draw=black,fill=mygreen] at (B1) {$b$};
\node[dot=9pt,draw=black,fill=myred] at (C1) {$c$};
\node[dot=9pt,draw=black,fill=mygreen] at (A2) {$a$};
\node[dot=9pt,draw=black,fill=mygreen] at (B2) {$b$};
\node[dot=9pt,draw=black,fill=myred] at (C2) {$c$};
\node[dot=9pt,draw=black,fill=mygreen] at (A3) {$a$};
\node[dot=9pt,draw=black,fill=mygreen] at (B3) {$b$};
\node[dot=9pt,draw=black,fill=mygreen] at (C3) {$c$};
\node[dot=9pt,draw=black,fill=myblue] at (A4) {$a$};
\node[dot=9pt,draw=black,fill=mygreen] at (B4) {$b$};
\node[dot=9pt,draw=black,fill=mygreen] at (C4) {$c$};
\end{tikzpicture}}

\newsavebox\KGone
\sbox\KGone{\begin{tikzpicture}[
dot/.style = {draw, circle, minimum size=#1,
              inner sep=0pt, outer sep=0pt},
dot/.default = 6pt
]
\scriptsize
\pgfmathsetmacro{\tline}{-0.25}
\pgfmathsetmacro{\w}{1.3}
\pgfmathsetmacro{\h}{1.5}
\pgfmathsetmacro{\inh}{0.7}
\pgfmathsetmacro{\dist}{1.9}
\pgfmathsetmacro{\Ax}{0.5}
\pgfmathsetmacro{\Ay}{1.4}
\pgfmathsetmacro{\Bx}{0.4}
\pgfmathsetmacro{\By}{0.5}
\pgfmathsetmacro{\Cx}{1}
\pgfmathsetmacro{\Cy}{1}


\foreach \x in {1,...,4}
{
\draw[fill=gray!90!black,opacity=0.2] (\x*\dist,0) -- (\x*\dist+\w,\inh) -- (\x*\dist+\w,\inh+\h) -- (\x*\dist,\h) -- cycle;


\node[dot=9pt,draw=none] (A\x) at (\x*\dist+\Ax, 0+\Ay) {};
\node[dot=9pt,draw=none] (B\x) at (\x*\dist+\Bx, 0+\By) {};
\node[dot=9pt,draw=none] (C\x) at (\x*\dist+\Cx, 0+\Cy) {};
}

\draw[<->,blue,thick] (A2) -- (B2) node[pos=0.5,left] {0};
\draw[<-,blue,thick] (A3) -- (B2) node[pos=0.15,above] {1};
\draw[<-,blue,thick] (B3) -- (A2) node[pos=0.12,below] {1};

\draw[<->,blue,thick] (B3) -- (C3) node[pos=0.7,left=0.05] {0};
\draw[<-,blue,thick] (A4) -- (B2) node[pos=0.17,above] {2};
\draw[<-,blue,thick] (B4) -- (A2) node[pos=0.1,below] {2};
\draw[blue,thick] (B4) -- (C3) node[pos=0.7,above] {1};
\draw[blue,thick] (C4) -- (B3) node[pos=0.3,above] {1};
\draw[<->,blue,thick] (A4) -- (C4) node[pos=0.8,above] {0};
\draw[<->,blue,thick] (B4) -- (C4) node[pos=0.7,below] {0};

\node[dot=9pt,draw=black,fill=myblue] at (A1) {$a$};
\node[dot=9pt,draw=black,fill=mygreen] at (B1) {$b$};
\node[dot=9pt,draw=black,fill=myred] at (C1) {$c$};
\node[dot=9pt,draw=black,fill=mygreen] at (A2) {$a$};
\node[dot=9pt,draw=black,fill=mygreen] at (B2) {$b$};
\node[dot=9pt,draw=black,fill=myred] at (C2) {$c$};
\node[dot=9pt,draw=black,fill=mygreen] at (A3) {$a$};
\node[dot=9pt,draw=black,fill=mygreen] at (B3) {$b$};
\node[dot=9pt,draw=black,fill=mygreen] at (C3) {$c$};
\node[dot=9pt,draw=black,fill=myblue] at (A4) {$a$};
\node[dot=9pt,draw=black,fill=mygreen] at (B4) {$b$};
\node[dot=9pt,draw=black,fill=mygreen] at (C4) {$c$};
\end{tikzpicture}}

\newsavebox\WLone
\sbox\WLone{\begin{tikzpicture}[
dot/.style = {draw, circle, minimum size=#1,
              inner sep=0pt, outer sep=0pt},
dot/.default = 6pt
]
\scriptsize
\pgfmathsetmacro{\tline}{-0.25}
\pgfmathsetmacro{\w}{1.3}
\pgfmathsetmacro{\h}{1.5}
\pgfmathsetmacro{\inh}{0.7}
\pgfmathsetmacro{\dist}{1.9}
\pgfmathsetmacro{\Ax}{0.5}
\pgfmathsetmacro{\Ay}{1.4}
\pgfmathsetmacro{\Bx}{0.4}
\pgfmathsetmacro{\By}{0.5}
\pgfmathsetmacro{\Cx}{1}
\pgfmathsetmacro{\Cy}{1}


\foreach \x in {1,...,4}
{
\draw[fill=gray!90!black,opacity=0.2] (\x*\dist,0) -- (\x*\dist+\w,\inh) -- (\x*\dist+\w,\inh+\h) -- (\x*\dist,\h) -- cycle;


\node[dot=9pt,draw=none] (A\x) at (\x*\dist+\Ax, 0+\Ay) {};
\node[dot=9pt,draw=none] (B\x) at (\x*\dist+\Bx, 0+\By) {};
\node[dot=9pt,draw=none] (C\x) at (\x*\dist+\Cx, 0+\Cy) {};
}

\draw[<->,blue,thick] (A2) -- (B2) node[pos=0.5,left] {0};
\draw[<-,blue,thick] (A3) -- (B2) node[pos=0.15,above] {1};
\draw[<-,blue,thick] (B3) -- (A2) node[pos=0.12,below] {1};

\draw[<->,blue,thick] (B3) -- (C3) node[pos=0.7,left=0.05] {0};
\draw[<-,blue,thick] (A4) -- (B2) node[pos=0.17,above] {2};
\draw[<-,blue,thick] (B4) -- (A2) node[pos=0.1,below] {2};
\draw[blue,thick] (B4) -- (C3) node[pos=0.7,above] {1};
\draw[blue,thick] (C4) -- (B3) node[pos=0.3,above] {1};
\draw[<->,blue,thick] (A4) -- (C4) node[pos=0.8,above] {0};
\draw[<->,blue,thick] (B4) -- (C4) node[pos=0.7,below] {0};

\node[dot=9pt,draw=black,fill=myblue] at (A1) {$a$};
\node[dot=9pt,draw=black,fill=mygreen] at (B1) {$b$};
\node[dot=9pt,draw=black,fill=myred] at (C1) {$c$};
\node[dot=9pt,draw=black,fill=PaleYellow] at (A2) {$a$};
\node[dot=9pt,draw=black,fill=PaleYellow] at (B2) {$b$};
\node[dot=9pt,draw=black,fill=myred] at (C2) {$c$};
\node[dot=9pt,draw=black,fill=Aqua] at (A3) {$a$};
\node[dot=9pt,draw=black,fill=Chartreuse] at (B3) {$b$};
\node[dot=9pt,draw=black,fill=Crimson] at (C3) {$c$};
\node[dot=9pt,draw=black,fill=gray] at (A4) {$a$};
\node[dot=9pt,draw=black,fill=yellow] at (B4) {$b$};
\node[dot=9pt,draw=black,fill=HotPink] at (C4) {$c$};
\end{tikzpicture}}

\newsavebox\KGtwo
\sbox\KGtwo{\begin{tikzpicture}[
dot/.style = {draw, circle, minimum size=#1,
              inner sep=0pt, outer sep=0pt},
dot/.default = 6pt
]
\scriptsize
\pgfmathsetmacro{\tline}{-0.25}
\pgfmathsetmacro{\w}{1.3}
\pgfmathsetmacro{\h}{1.5}
\pgfmathsetmacro{\inh}{0.7}
\pgfmathsetmacro{\dist}{1.9}
\pgfmathsetmacro{\Ax}{0.5}
\pgfmathsetmacro{\Ay}{1.4}
\pgfmathsetmacro{\Bx}{0.4}
\pgfmathsetmacro{\By}{0.5}
\pgfmathsetmacro{\Cx}{1}
\pgfmathsetmacro{\Cy}{1}


\foreach \x in {1,...,4}
{
\draw[fill=gray!90!black,opacity=0.2] (\x*\dist,0) -- (\x*\dist+\w,\inh) -- (\x*\dist+\w,\inh+\h) -- (\x*\dist,\h) -- cycle;


\node[dot=9pt,draw=none] (A\x) at (\x*\dist+\Ax, 0+\Ay) {};
\node[dot=9pt,draw=none] (B\x) at (\x*\dist+\Bx, 0+\By) {};
\node[dot=9pt,draw=none] (C\x) at (\x*\dist+\Cx, 0+\Cy) {};
}

\draw[<->,blue,thick] (A2) -- (B2) node[midway,left] {0};
\draw[<->,blue,thick] (A3) -- (B3) node[midway,left] {1};
\draw[<->,blue,thick] (B3) -- (C3) node[pos=0.3,above] {0};
\draw[<->,blue,thick] (A4) -- (B4) node[midway,left] {2};
\draw[<->,blue,thick] (A4) -- (C4) node[pos=-0.08,right=0.07] {0};
\draw[<->,blue,thick] (B4) -- (C4) node[pos=0.3,above] {1};
\draw (B4) edge[<->,blue,thick,bend right=20] node[pos=0.7,below] {0} (C4);

\node[dot=9pt,draw=black,fill=myblue] at (A1) {$a$};
\node[dot=9pt,draw=black,fill=mygreen] at (B1) {$b$};
\node[dot=9pt,draw=black,fill=myred] at (C1) {$c$};
\node[dot=9pt,draw=black,fill=mygreen] at (A2) {$a$};
\node[dot=9pt,draw=black,fill=mygreen] at (B2) {$b$};
\node[dot=9pt,draw=black,fill=myred] at (C2) {$c$};
\node[dot=9pt,draw=black,fill=mygreen] at (A3) {$a$};
\node[dot=9pt,draw=black,fill=mygreen] at (B3) {$b$};
\node[dot=9pt,draw=black,fill=mygreen] at (C3) {$c$};
\node[dot=9pt,draw=black,fill=myblue] at (A4) {$a$};
\node[dot=9pt,draw=black,fill=mygreen] at (B4) {$b$};
\node[dot=9pt,draw=black,fill=mygreen] at (C4) {$c$};
\end{tikzpicture}}

\newsavebox\WLtwo
\sbox\WLtwo{\begin{tikzpicture}[
dot/.style = {draw, circle, minimum size=#1,
              inner sep=0pt, outer sep=0pt},
dot/.default = 6pt
]
\scriptsize
\pgfmathsetmacro{\tline}{-0.25}
\pgfmathsetmacro{\w}{1.3}
\pgfmathsetmacro{\h}{1.5}
\pgfmathsetmacro{\inh}{0.7}
\pgfmathsetmacro{\dist}{1.9}
\pgfmathsetmacro{\Ax}{0.5}
\pgfmathsetmacro{\Ay}{1.4}
\pgfmathsetmacro{\Bx}{0.4}
\pgfmathsetmacro{\By}{0.5}
\pgfmathsetmacro{\Cx}{1}
\pgfmathsetmacro{\Cy}{1}


\foreach \x in {1,...,4}
{
\draw[fill=gray!90!black,opacity=0.2] (\x*\dist,0) -- (\x*\dist+\w,\inh) -- (\x*\dist+\w,\inh+\h) -- (\x*\dist,\h) -- cycle;


\node[dot=9pt,draw=none] (A\x) at (\x*\dist+\Ax, 0+\Ay) {};
\node[dot=9pt,draw=none] (B\x) at (\x*\dist+\Bx, 0+\By) {};
\node[dot=9pt,draw=none] (C\x) at (\x*\dist+\Cx, 0+\Cy) {};
}

\draw[<->,blue,thick] (A2) -- (B2) node[midway,left] {0};
\draw[<->,blue,thick] (A3) -- (B3) node[midway,left] {1};
\draw[<->,blue,thick] (B3) -- (C3) node[pos=0.3,above] {0};
\draw[<->,blue,thick] (A4) -- (B4) node[midway,left] {2};
\draw[<->,blue,thick] (A4) -- (C4) node[pos=-0.08,right=0.07] {0};
\draw[<->,blue,thick] (B4) -- (C4) node[pos=0.3,above] {1};
\draw (B4) edge[<->,blue,thick,bend right=20] node[pos=0.7,below] {0} (C4);

\node[dot=9pt,draw=black,fill=myblue] at (A1) {$a$};
\node[dot=9pt,draw=black,fill=mygreen] at (B1) {$b$};
\node[dot=9pt,draw=black,fill=myred] at (C1) {$c$};
\node[dot=9pt,draw=black,fill=PaleYellow] at (A2) {$a$};
\node[dot=9pt,draw=black,fill=PaleYellow] at (B2) {$b$};
\node[dot=9pt,draw=black,fill=myred] at (C2) {$c$};
\node[dot=9pt,draw=black,fill=Aqua] at (A3) {$a$};
\node[dot=9pt,draw=black,fill=Chartreuse] at (B3) {$b$};
\node[dot=9pt,draw=black,fill=Crimson] at (C3) {$c$};
\node[dot=9pt,draw=black,fill=gray] at (A4) {$a$};
\node[dot=9pt,draw=black,fill=yellow] at (B4) {$b$};
\node[dot=9pt,draw=black,fill=HotPink] at (C4) {$c$};
\end{tikzpicture}}

\begin{figure*}
    \begin{tikzpicture}

        \node[draw,inner sep=4pt] at (0,0) (T1) {\scalebox{0.65}{\usebox{\Tempgraph}}};

        \node[draw,inner sep=4pt] at (5.5,1.5) (KG1) {\scalebox{0.65}{\usebox{\KGone}}};
        \node[draw,inner sep=4pt, label=above:{Nodes distinguishable by global \MPTGNN{}s:}] at (12.4,1.5) (WL1) {\scalebox{0.65}{\usebox{\WLone}}};
        
        \node[draw,inner sep=4pt] at (5.5,-1.5) (KG2) {\scalebox{0.65}{\usebox{\KGtwo}}};        
        \node[draw,inner sep=4pt, label=above:{Nodes distinguishable by local \MPTGNN{}s:}] at (12.4,-1.5) (WL2) {\scalebox{0.65}{\usebox{\WLtwo}}};

        \draw[->, bend left=25, thick, >=latex, scale=20] (T1.north) to node[above] {Construct $\Kone{\TG}$} (KG1.west);
        \draw[->, bend left=15, thick, >=latex, scale=20] (KG1.east) to node[above] {Apply 1-WL} (WL1.west);

        \draw[->, bend right=25, thick, >=latex, scale=20]  (T1.south) to node[below] {Construct $\Ktwo{\TG}$} (KG2.west);
        \draw[->, bend right=15, thick, >=latex, scale=20]  (KG2.east) to node[below] {Apply 1-WL} (WL2.west);
        
        
        

    \node at (-1.5,1.1) {$\TG$:};
        
    \end{tikzpicture}
\caption{
Our approach to determine which nodes in a temporal graph $\TG$ are distinguishable by \MPTGNN{}: we construct of knowledge graphs
$\Kone{\TG}$ and $\Ktwo{\TG}$, and then apply 1-WL}
\label{fig:transform}    
\end{figure*}
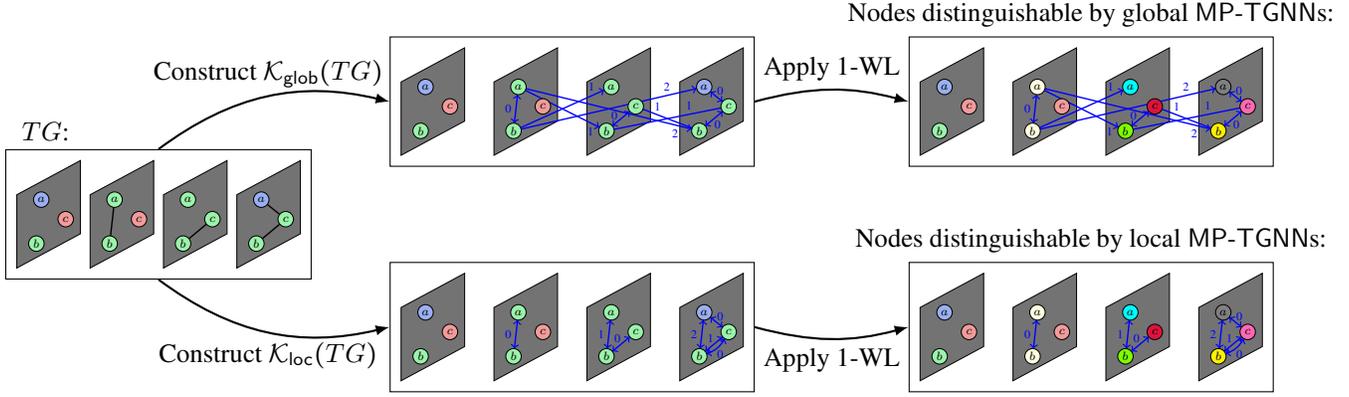

\emph{Message-passing graph neural networks} (or \MPGNN{}s) \cite{DBLP:conf/icml/GilmerSRVD17}
are prominent models for graph learning, which have achieved state-of-the art 
performance
in tasks of link prediction as well as in node  and graph classification.
Importantly, they proved successful in a number of real-world applications including social networks, protein-protein interactions, and knowledge graphs \cite{DBLP:journals/aiopen/ZhouCHZYLWLS20}.




In recent years, there has been growing interest in adapting \MPGNN{}s to process \emph{temporal graphs} (for an example of a temporal graph see the left part of \Cref{fig:transform}) which are particularly well-suited for dynamic applications such as recommender systems \cite{DBLP:journals/csur/WuSZXC23}, 
traffic forecasting \cite{DBLP:conf/ijcai/YuYZ18}, finance networks \cite{DBLP:conf/aaai/ParejaDCMSKKSL20}, and modelling the spread of diseases \cite{DBLP:journals/corr/abs-2007-03113}.
Research in this direction gave rise to various temporal \MPGNN{}s (\MPTGNN{}s) obtained by introducing temporal variants of the message-passing mechanism
\cite{longa2023graph,DBLP:journals/access/SkardingGM21,DBLP:conf/icml/Gao022}.
This can be obtained by assigning to
graph nodes different embeddings (feature vectors)
for different time points and then passing messages between timestamped nodes.
Depending on the routes of messages-passing  between timestamped nodes and on the encoding of the temporal component in the messages, 
we arrive at various temporal message-passing mechanisms.
In this paper we distinguish two main groups of \MPTGNN{}s:
\emph{global}, 
where messages can be passed between nodes stamped with different times
and
\emph{local}, 
where messages are passed only between nodes stamped with the same time, 
while information about other times is encoded within messages.
%
%

Although several variants of global \cite{longa2023graph,DBLP:conf/iclr/XuRKKA20,DBLP:conf/log/LuoL22} and local \cite{DBLP:journals/corr/abs-2006-10637,DBLP:conf/www/QuZDS20} \MPTGNN{}s have been designed and successfully applied, 
we still lack a good understanding of modelling capabilities
dictated by their temporal message-passing mechanisms, and do not have answers to the following fundamental questions. 
What tools can we use to analyse the expressive power of global and local \MPTGNN{}s?
Are there limits to the expressive power of either type? 
Which type can express more? 
How does the difference in expressiveness affect practical performance?
Answers to these questions
are key when choosing an appropriate temporal message-passing mechanism for a particular task and when designing new \MPTGNN{}s. 
The importance of answering such questions has been clearly shown by research on expressive power of static \MPGNN{}s, which 
equipped us with powerful tools and  
gave rise to a whole new research direction
\cite{DBLP:conf/aaai/0001RFHLRG19,DBLP:conf/iclr/XuHLJ19,DBLP:journals/combinatorica/CaiFI92,DBLP:conf/lics/Grohe23,DBLP:conf/iclr/BarceloKM0RS20}.
In the temporal setting, however, such an analysis is still missing.  We aim to fill this urgent gap.


\paragraph{Contributions.}
Our main contributions are as follows:
\begin{itemize}
\item We formalise the two main types of \MPTGNN{}s, global and local, depending on the form of the adopted temporal message-passing mechanism.  

\item We characterise the expressive power of both types.
To determine which temporal nodes can be distinguished by \MPTGNN{}s, we construct a knowledge graph and then apply the 1-dimensional Weisfeiler-Leman test (1-WL).
As depicted in \Cref{fig:transform}, our construction of the knowledge graph is different for global and local \MPTGNN{}s, but in both cases this approach allows us to precisely capture the expressive power of \MPTGNN{}s.
For example, given the temporal graph $\TG$ in \Cref{fig:transform}, global and local \MPTGNN{}s distinguish the same nodes, since the colourings in the rightmost graphs are the same.

\item We use the above characterisation to show that, quite surprisingly, both global and local \MPTGNN{}s can distinguish nodes which are pointwise isomorphic.
This leads us to introduce a stronger \emph{timewise isomorphism}, well-suited for characterisation of \MPTGNN{}s.

\item The Weisfeiler-Leman characterisation also allows us to show that global and local \MPTGNN{}s have incomparable expressive power: each of the types can distinguish nodes which are indistinguishable by the other type.
However, if the input temporal graph is 
\emph{colour-persistent} (initial embedding of each node is the same at all time points),  local \MPTGNN{}s are  more expressive than global \MPTGNN{}s.
We can extend these results to  a complete expressiveness classification as in \Cref{fig:express}.

\begin{figure}[ht]
\centering
\begin{tikzpicture}

\node[draw=black, text width=3cm] (one) at (0,0) {global \MPTGNN{}s, \\ any \TG{}s};
\node[draw=black, text width=3cm]   (two) at (0,-1.5) {local \MPTGNN{}s, \\ any \TG{}s};
\node[draw=black, text width=3.2cm]  (three) at (4,0) {global \MPTGNN{}s, \\ colour-persistent \TG{}s};
\node[draw=black, text width=3.2cm] (four) at (4,-1.5) {local \MPTGNN{}s, \\ colour-persistent \TG{}s};

\path (one) -- (three) node[midway]  {$>$};
\path (one) -- (two) node[midway] {\tikz[baseline]{\node[rotate=90] {$\not\leq$};} {\tikz[baseline]{\node[rotate=90] {$\not\geq$};}}};
\path (two) -- (four) node[midway] {$=$};
\path (three) -- (four) node[midway]  [rotate=90] {$>$};

\end{tikzpicture}
\caption{Relative expressive power of \MPTGNN{}s}
\label{fig:express}
\end{figure}
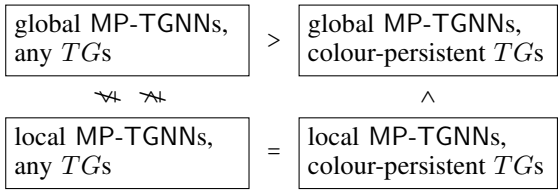

\item 
Finally, we experimentally validate our theoretical results by constructing proof-of-concept global and local models.
We show that, indeed, on colour-persistent graphs 
local models outperform global models, when compared on the temporal link-prediction and TGB 2.0 benchmark suite~\cite{gastinger2024tgb}.
This is the case when models use the same number of layers, and the difference in performance increases further if we choose optimal number of layers for each type of model separately.
\end{itemize}

\section{Background}

\paragraph{Temporal graphs.}
We focus on temporal graphs in the so-called \emph{snapshot} representation \cite{longa2023graph,DBLP:conf/icml/Gao022,DBLP:journals/access/SkardingGM21} shown in \Cref{fig:snapshot}. A  \emph{temporal graph} is
a finite sequence $\TG = (\G_1,t_1), \dots, (\G_n,t_n)$ of undirected node-coloured graphs $\G_i=(V_i,E_i,c_i)$, where $t_1 <  \dots < t_n$ are real-valued
\emph{time points}, constituting the temporal domain $\T(\TG)$.
Each $V_i$ is a finite set of nodes, each $E_i \subseteq \{ \{u,v\} \subseteq V_i \mid  u \neq v \}$ is a set of undirected edges, and
$c_i:V_i \to \D$ assigns nodes colours from some set $\D$, which could be real feature vectors.
Following standard notation, we sometimes use $\x_v(t_i)$ instead of $c_{i}(v)$.  We represent $c_{i}$ using different colours for nodes in figures.  
We assume that the domain of nodes does not change over time,  so $V_1= \ldots = V_n = V(\TG)$.
We call a pair of a node $v \in V(\TG)$ and a time point $t \in \T(\TG)$ a \emph{timestamped node} $(v,t)$ and we let $\TV(\TG)$ be the set of all timestamped nodes in $\TG$.
For the sake of a clear presentation we assume that edges are not labelled.

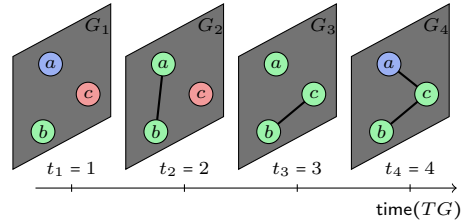
\begin{figure}[ht]
\centering
\begin{tikzpicture}[
dot/.style = {draw, circle, minimum size=#1,
              inner sep=0pt, outer sep=0pt},
dot/.default = 6pt
]
\scriptsize
\pgfmathsetmacro{\tline}{-0.25}
\pgfmathsetmacro{\w}{1.3}
\pgfmathsetmacro{\h}{1.5}
\pgfmathsetmacro{\inh}{0.7}
\pgfmathsetmacro{\dist}{1.5}
\pgfmathsetmacro{\Ax}{0.5}
\pgfmathsetmacro{\Ay}{1.4}
\pgfmathsetmacro{\Bx}{0.4}
\pgfmathsetmacro{\By}{0.5}
\pgfmathsetmacro{\Cx}{1}
\pgfmathsetmacro{\Cy}{1}

\draw[->] (1.8,\tline) -- (7.4,\tline);
\node at (6.9,\tline-0.3) {$\T(\TG)$};

\foreach \x in {1,...,4}
{
\draw[fill=gray!90!black,opacity=0.2] (\x*\dist,0) -- (\x*\dist+\w,\inh) -- (\x*\dist+\w,\inh+\h) -- (\x*\dist,\h) -- cycle;

\node at (\x*\dist + 0.87 * \w, \inh+ \h - 0.3) {$G_{\x}$};
\node at (\x*\dist + 0.6 * \w, 0) {$t_{\x}=\x$};
\draw[-] (\x*\dist + 0.6 * \w, \tline+0.05) -- (\x*\dist + 0.6 * \w, \tline-0.05);

\node (A\x) at (\x*\dist+\Ax, 0+\Ay) {};
\node (B\x) at (\x*\dist+\Bx, 0+\By) {};
\node (C\x) at (\x*\dist+\Cx, 0+\Cy) {};
}

\draw[thick] (A2) -- (B2);
\draw[thick](B3) -- (C3);
\draw[thick] (A4) -- (C4);
\draw[thick] (B4) -- (C4);

\node[dot=9pt,draw=black,fill=myblue] at (A1) {$a$};
\node[dot=9pt,draw=black,fill=mygreen] at (B1) {$b$};
\node[dot=9pt,draw=black,fill=myred] at (C1) {$c$};
\node[dot=9pt,draw=black,fill=mygreen] at (A2) {$a$};
\node[dot=9pt,draw=black,fill=mygreen] at (B2) {$b$};
\node[dot=9pt,draw=black,fill=myred] at (C2) {$c$};
\node[dot=9pt,draw=black,fill=mygreen] at (A3) {$a$};
\node[dot=9pt,draw=black,fill=mygreen] at (B3) {$b$};
\node[dot=9pt,draw=black,fill=mygreen] at (C3) {$c$};
\node[dot=9pt,draw=black,fill=myblue] at (A4) {$a$};
\node[dot=9pt,draw=black,fill=mygreen] at (B4) {$b$};
\node[dot=9pt,draw=black,fill=mygreen] at (C4) {$c$};
\end{tikzpicture}
\caption{A temporal graph  in the snapshot representation}
\label{fig:snapshot}
\end{figure}

\noindent We say that a temporal graph is \emph{colour-persistent} if initial colours of  nodes do not change in time, so $c_{i}(v) = c_{j}(v)$ for each node $v$ and all  $i,j \in \{ 1, \dots, n\}$, see \Cref{fig:aggregated}~(\subref{fig:snap}).
Colour-persistent graphs can be also represented
as  static edge-labelled multi graphs, called \emph{aggregated} form~\cite{DBLP:conf/icml/Gao022},  as depicted in \Cref{fig:aggregated}~(\subref{fig:aggr}).

\begin{figure}[ht]
\centering
\begin{subfigure}[b]{0.35\textwidth}
\begin{tikzpicture}[
dot/.style = {draw, circle, minimum size=#1,
              inner sep=0pt, outer sep=0pt},
dot/.default = 6pt
]
\scriptsize
\pgfmathsetmacro{\tline}{-0.25}
\pgfmathsetmacro{\w}{1.3}
\pgfmathsetmacro{\h}{1.5}
\pgfmathsetmacro{\inh}{0.7}
\pgfmathsetmacro{\dist}{1.5}
\pgfmathsetmacro{\Ax}{0.5}
\pgfmathsetmacro{\Ay}{1.4}
\pgfmathsetmacro{\Bx}{0.4}
\pgfmathsetmacro{\By}{0.5}
\pgfmathsetmacro{\Cx}{1}
\pgfmathsetmacro{\Cy}{1}

\draw[->] (1.8,\tline) -- (7.4,\tline);

\foreach \x in {1,...,4}
{
\draw[fill=gray!90!black,opacity=0.2] (\x*\dist,0) -- (\x*\dist+\w,\inh) -- (\x*\dist+\w,\inh+\h) -- (\x*\dist,\h) -- cycle;

\node at (\x*\dist + 0.87 * \w, \inh+ \h - 0.3) {$G_{\x}$};
\node at (\x*\dist + 0.6 * \w, 0) {$t_{\x}$};
\draw[-] (\x*\dist + 0.6 * \w, \tline+0.05) -- (\x*\dist + 0.6 * \w, \tline-0.05);

\node (A\x) at (\x*\dist+\Ax, 0+\Ay) {};
\node (B\x) at (\x*\dist+\Bx, 0+\By) {};
\node (C\x) at (\x*\dist+\Cx, 0+\Cy) {};
}

\draw[thick] (A2) -- (B2);
\draw[thick](B3) -- (C3);
\draw[thick] (A4) -- (C4);
\draw[thick] (B4) -- (C4);

\node[dot=9pt,draw=black,fill=myblue] at (A1) {$a$};
\node[dot=9pt,draw=black,fill=mygreen] at (B1) {$b$};
\node[dot=9pt,draw=black,fill=mygreen] at (C1) {$c$};
\node[dot=9pt,draw=black,fill=myblue] at (A2) {$a$};
\node[dot=9pt,draw=black,fill=mygreen] at (B2) {$b$};
\node[dot=9pt,draw=black,fill=mygreen] at (C2) {$c$};
\node[dot=9pt,draw=black,fill=myblue] at (A3) {$a$};
\node[dot=9pt,draw=black,fill=mygreen] at (B3) {$b$};
\node[dot=9pt,draw=black,fill=mygreen] at (C3) {$c$};
\node[dot=9pt,draw=black,fill=myblue] at (A4) {$a$};
\node[dot=9pt,draw=black,fill=mygreen] at (B4) {$b$};
\node[dot=9pt,draw=black,fill=mygreen] at (C4) {$c$};
\end{tikzpicture}
\caption{}
\label{fig:snap}
\end{subfigure}
\rulesep
\begin{subfigure}[b]{0.09\textwidth}
\begin{tikzpicture}[scale=1,
dot/.style = {draw, circle, minimum size=#1,
              inner sep=0pt, outer sep=0pt},
dot/.default = 6pt
]
\scriptsize
\pgfmathsetmacro{\tline}{-0.25}
\pgfmathsetmacro{\w}{1.3}
\pgfmathsetmacro{\h}{1.5}
\pgfmathsetmacro{\inh}{0.7}
\pgfmathsetmacro{\dist}{1.7}
\pgfmathsetmacro{\Ax}{0.5}
\pgfmathsetmacro{\Ay}{1.4}
\pgfmathsetmacro{\Bx}{0.4}
\pgfmathsetmacro{\By}{0.5}
\pgfmathsetmacro{\Cx}{1}
\pgfmathsetmacro{\Cy}{1}

\node[white] at (2.4,-0.16) {$(\TG)$};

\foreach \x in {1}
{

\node (A\x) at (\x*\dist+\Ax, 0+\Ay) {};
\node (B\x) at (\x*\dist+\Bx, 0+\By) {};
\node (C\x) at (\x*\dist+\Cx, 0+\Cy) {};
}

\draw[thick] (A1) -- (B1) node[pos=0.5,left=0.01] {$t_2$};
\draw[thick] (A1) -- (C1) node[pos=0.9,above=0.05] {$t_4$};
\draw[thick] (B1) -- (C1) node[pos=0.4,above=0.05] {$t_3$};
\draw (B1) edge[thick,bend right=20] node[pos=0.9,below=0.1] {$t_4$} (C1);

\node[dot=9pt,draw=black,fill=myblue] at (A1) {$a$};
\node[dot=9pt,draw=black,fill=mygreen] at (B1) {$b$};
\node[dot=9pt,draw=black,fill=mygreen] at (C1) {$c$};

\end{tikzpicture}
\caption{}
\label{fig:aggr}
\end{subfigure}
\caption{A colour-persistent temporal graph~(\subref{fig:snap}) and its  aggregated representation~(\subref{fig:aggr})}
\label{fig:aggregated}
\end{figure}
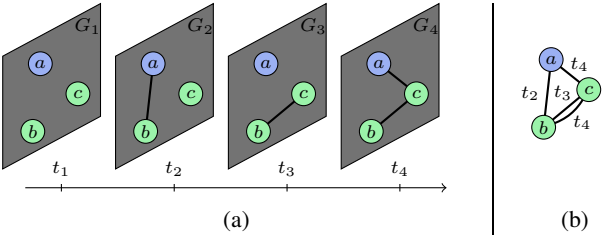


\paragraph{Temporal graph neural networks with message-passing.}
We let a \emph{message-passing temporal graph neural network}
(\MPTGNN) be a model $\model$ which,
given a temporal graph $\TG$,
computes embeddings for all timestamped nodes by implementing a temporal variant of message-passing.
Embeddings are then used to predict links or classify nodes and graphs.
Some models (e.g. TGAT, NAT, TGN) apply temporal message-passing to arbitrary temporal graphs, whereas others (e.g. TDGNN) are applicable to colour-persistent temporal graphs (or equivalently, to  the aggregated representation) only.
%
%
Below we present a general form of an  \MPTGNN{} model $\model$ with $L$ layers, which subsumes a number of message-passing mechanisms.
Given a temporal graph $\TG = (\G_1,t_1), \dots, (\G_n,t_n)$,  a model $\model$ computes for each node $v$,  time point $t$, and   layer $\ell \in \{0,\ldots,L\}$
an embedding $\h_v^{(\ell)}(t)$ as follows:
\begin{align}
\h_v^{(0)}(t) &= \x_v(t), \label{TGN1}
\\
\h_v^{(\ell)}(t) &= \com^{(\ell)} \Big(
\h_v^{(\ell-1)}(t),
\agg^{(\ell)} \Big(
\label{TGN2}
\\ 
&  \quad \ldblbrace
( \fone, g(t-t') ) \mid
 (u,t') \in \NH(v,t) 
\rdblbrace
\Big)
\Big), \nonumber
\end{align}
where:
\begin{itemize}
\item $\com^{(\ell)}$ and $\agg^{(\ell)}$ are \emph{combination} and \emph{aggregation} functions in layer $\ell$; $\com^{(\ell)}$ maps a pair of vectors into a single vector, whereas $\agg^{(\ell)}$ maps a multiset, represented as $\ldblbrace \cdots \rdblbrace$, into a single vector,
\item $g$ maps a time duration into a vector or scalar quantity,
\item $\NH(v,t)$ is the temporal neighbourhood of $(v,t)$,  defined as follows \cite{DBLP:conf/nips/SouzaMKG22}:
\begin{multline}
\NH(v,t) = \Big\{ (u,t') \mid 
t' = t_i 
\text{ and }
\{ u,v \} \in E_{i} ,
\\ 
\text{ for some } (\G_i,t_i) \in \TG \text{ with } t_i \leq t
  \Big\}.
  \label{eq:NH} 
\end{multline}
Hence, $\NH(v,t)$ is
the set of timestamped nodes $(u,t')$ such that there is an edge between $u$ and $v$ at  $t' \leq t$.

\item $\fone$ is either $\mathbf{h}_u^{(\ell-1)}(t')$ or $\mathbf{h}_u^{(\ell-1)}(t)$.
If $\fone = \mathbf{h}_u^{(\ell-1)}(t')$ we say that $\model$ is a \emph{\glob} (in time) \TGNN{}, as computation of $\h_v^{(\ell)}(t)$  requires aggregation of embeddings in all past time points $t'$; 
\emph{\glob} \TGNN{}s  are also called  \emph{Temporal Embedding} \TGNN{}s~\cite{longa2023graph} and include  TGAT and NAT.
If $\fone = \mathbf{h}_u^{(\ell-1)}(t)$ we say that $\model$ is  \emph{\loc}, as only embeddings from the current time point $t$ are aggregated;
 \emph{\loc} \TGNN{}s
 include  TGN and TDGNN.
\end{itemize}

\paragraph{Weisfeiler-Leman algorithm.}
An \emph{isomorphism} between undirected node-coloured graphs $\G_1=(V_1,E_1,c_1)$ and $\G_2=(V_2,E_2,c_2)$
is any bijection $f:V_1 \to V_2$, 
satisfying for any $u$ and $v$:
(i)~$c_1(v) = c_2( f (v) )$
and 
(ii)~$\{u,v\} \in E_1$ if and only if 
$\{f(u),f(v)\} \in E_2$.
The \emph{1-dimensional Weisfeiler-Leman algorithm} (1-WL) \cite{weisfeiler1968reduction}
is a powerful heuristic for graph isomorphism~\cite{DBLP:conf/focs/BabaiK79}, which has  the same expressive power as \MPGNN{}s with injective aggregation and combination
\cite{DBLP:conf/aaai/0001RFHLRG19,DBLP:conf/iclr/XuHLJ19}.

Recently 1-WL has been  applied to \emph{knowledge graphs} $\KG = (V, E, R, c)$
where $V$ are nodes, $E \subseteq R \times V \times V$ are directed edges with  labels from $R$, and $c:V \to \D$ colours nodes \cite{DBLP:conf/nips/Huang0CB23,DBLP:conf/log/Barcelo00O22}.
An \emph{isomorphism} between knowledge graphs $\KG_1=(V_1,E_1,R_1,c_1)$ and $\KG_2=(V_2,E_2,R_2,c_2)$ is  any bijection  $f:V_1 \to V_2$ 
such that, for all $u,v \in V_1$ and $r\in R_1$:
(i)~$c_1(v) = c_2( f (v) )$
and 
(ii)~$(r,u,v) \in E_1$ if and only if 
$(r,f(u),f(v)) \in E_2$.
A \emph{relational local 1-WL algorithm}
(conventionally $\rwl_1$, but we write \rwl{}) is a natural extension of 1-WL to the case of knowledge graphs~\cite{DBLP:conf/nips/Huang0CB23}.
Given a knowledge graph $\KG = (V, E, R, c)$,
the algorithm
computes iteratively, for all $v\in V$ and
 $\ell \in \N$, values $\rwl^{(\ell)}(v)$ as follows:
\begin{align*}
\rwl^{(0)}(v) &=c(v),
\\
\rwl^{(\ell)}(v) &= \tau \Big( \rwl^{(\ell-1)}(v), 
\\
& \qquad 
\ldblbrace ( \rwl^{(\ell-1)}(u) ,r) \mid u \in \NH_{r}(v), r \in R \rdblbrace   \Big),
\end{align*}
where
$\NH_{r}(v) = \{ u \mid (r,u, v) \in E \}$
is the \emph{$r$-neighbourhood} of $v$,
and $\tau$ is an injective function.
It is shown that \rwl{} has the same expressive power as R-MPNNs, that is,
\MPGNN{}s processing knowledge graphs~\cite{DBLP:conf/nips/Huang0CB23}.

\section{Related Work}
There is recently an increasing interest in temporal and dynamic graph neural networks \cite{longa2023graph,DBLP:journals/csur/QinY24,DBLP:journals/access/SkardingGM21,DBLP:journals/jmlr/KazemiGJKSFP20}.
Pertinent models include
TGN~\cite{DBLP:journals/corr/abs-2006-10637}, TGAT~\cite{DBLP:conf/iclr/XuRKKA20}, TDGNN~\cite{DBLP:conf/www/QuZDS20},  
and NAT~\cite{DBLP:conf/log/LuoL22},
which are all based on temporal message-passing mechanisms.

Expressive power results for temporal models are very limited.
\citet{DBLP:conf/nips/SouzaMKG22} compared expressive power of temporal graph neural networks exploiting temporal walks, with those based on local message passing combined with recurrent memory modules.
\citet{DBLP:conf/icml/Gao022} compared
time-and-graph with time-then-graph models, which are obtained by different combinations of static graph neural networks and recurrent neural networks.
In the context of temporal knowledge graphs, expressive power of similar models was recently considered by  \citet{DBLP:conf/nips/ChenW23}.

More mature results have been established for models processing edge-labelled graphs.
Such graphs are closely related to temporal graphs, since the aggregated representation of a temporal graph \cite{DBLP:conf/icml/Gao022}, \Cref{fig:aggregated}~(\subref{fig:aggr}), 
is a multigraph with edges labelled by time points.
However, since the aggregated representation does not allow us to assign different colours to the same node in different time points,
not all temporal graphs can be directly
transformed into multigraphs.
\citet{DBLP:conf/log/Barcelo00O22} introduced 1-WL for models processing undirected multi-relational graphs, whereas \citet{DBLP:journals/nn/BeddarWiesingDGLMST24} introduced 1-WL for dynamic graphs.
\citet{DBLP:conf/nips/Huang0CB23} proposed 1-WL for models processing directed multi-relational graphs (i.e. knowledge graphs), namely for
relational message passing neural networks (R-MPNNs), which encompass several known models such as RGCN \cite{DBLP:conf/esws/SchlichtkrullKB18} and CompGCN \cite{DBLP:conf/iclr/VashishthSNT20}.

Temporal graphs can be also given in the \emph{event-based} representation \cite{longa2023graph},  as a sequence of timestamped events that add/delete edges or modify feature vectors of nodes. 
Since temporal graphs in the aggregated and event-based representations can be transformed into the snapshot representation~\cite{DBLP:conf/icml/Gao022,longa2023graph}, we 
focus on the snapshot representation in the paper.

\section{Temporal Weisfeiler-Leman Characterisation}\label{sec:WL}

We provide a general approach for establishing expressive power of \MPTGNN{}s using standard 1-WL.
To do so, we transform a temporal graph $\TG$ into a knowledge graph $\KG$ such that
\MPTGNN{}s can distinguish exactly those nodes in $\TG$ whose counterparts in $\KG$ can be distinguished by the standard 1-WL.
This contrasts with approaches studying expressive power by modifying 1-WL for particular types of temporal graph neural networks \cite{DBLP:conf/nips/SouzaMKG22,DBLP:conf/icml/Gao022}.
Note that our results concern distinguishability of nodes, not graphs. Node distinguishability is likely of more practical interest and can be  used to distinguish graphs.

We transform $\TG$ into two knowledge graphs: $\Kone{\TG}$ and $\Ktwo{\TG}$,  suitable for analysing, respectively, global and local \MPTGNN{}s.
We first introduce $\Kone{\TG}$, whose 
edges correspond to temporal message-passing in global \MPTGNN{}s (\Cref{fig:Kone}).
Intuitively,  $\Kone{\TG}$ contains a separate node $(v,t)$ for each timestamped node in $\TG$ and an edge between $(v,t)$ and $(u,t')$ labelled by $t-t'$ if $(u,t')$ is in the temporal neighbourhood of $(v,t)$.

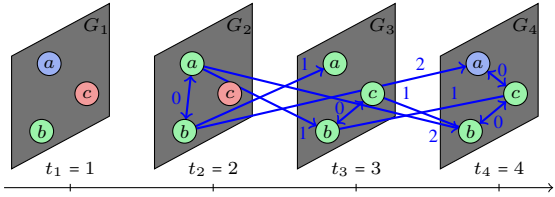
\begin{figure}[ht]
\centering
\begin{tikzpicture}[
dot/.style = {draw, circle, minimum size=#1,
              inner sep=0pt, outer sep=0pt},
dot/.default = 6pt
]
\scriptsize
\pgfmathsetmacro{\tline}{-0.25}
\pgfmathsetmacro{\w}{1.3}
\pgfmathsetmacro{\h}{1.5}
\pgfmathsetmacro{\inh}{0.7}
\pgfmathsetmacro{\dist}{1.9}
\pgfmathsetmacro{\Ax}{0.5}
\pgfmathsetmacro{\Ay}{1.4}
\pgfmathsetmacro{\Bx}{0.4}
\pgfmathsetmacro{\By}{0.5}
\pgfmathsetmacro{\Cx}{1}
\pgfmathsetmacro{\Cy}{1}

\draw[->] (1.8,\tline) -- (9.1,\tline);

\foreach \x in {1,...,4}
{
\draw[fill=gray!90!black,opacity=0.2] (\x*\dist,0) -- (\x*\dist+\w,\inh) -- (\x*\dist+\w,\inh+\h) -- (\x*\dist,\h) -- cycle;

\node at (\x*\dist + 0.87 * \w, \inh+ \h - 0.3) {$G_{\x}$};
\node at (\x*\dist + 0.6 * \w, 0) {$t_{\x}={\x}$};
\draw[-] (\x*\dist + 0.6 * \w, \tline+0.05) -- (\x*\dist + 0.6 * \w, \tline-0.05);

\node[dot=9pt,draw=none] (A\x) at (\x*\dist+\Ax, 0+\Ay) {};
\node[dot=9pt,draw=none] (B\x) at (\x*\dist+\Bx, 0+\By) {};
\node[dot=9pt,draw=none] (C\x) at (\x*\dist+\Cx, 0+\Cy) {};
}

\draw[<->,blue,thick] (A2) -- (B2) node[pos=0.5,left] {0};
\draw[<-,blue,thick] (A3) -- (B2) node[pos=0.15,above] {1};
\draw[<-,blue,thick] (B3) -- (A2) node[pos=0.12,below] {1};

\draw[<->,blue,thick] (B3) -- (C3) node[pos=0.7,left=0.05] {0};
\draw[<-,blue,thick] (A4) -- (B2) node[pos=0.17,above] {2};
\draw[<-,blue,thick] (B4) -- (A2) node[pos=0.1,below] {2};
\draw[blue,thick] (B4) -- (C3) node[pos=0.7,above] {1};
\draw[blue,thick] (C4) -- (B3) node[pos=0.3,above] {1};
\draw[<->,blue,thick] (A4) -- (C4) node[pos=0.8,above] {0};
\draw[<->,blue,thick] (B4) -- (C4) node[pos=0.7,below] {0};

\node[dot=9pt,draw=black,fill=myblue] at (A1) {$a$};
\node[dot=9pt,draw=black,fill=mygreen] at (B1) {$b$};
\node[dot=9pt,draw=black,fill=myred] at (C1) {$c$};
\node[dot=9pt,draw=black,fill=mygreen] at (A2) {$a$};
\node[dot=9pt,draw=black,fill=mygreen] at (B2) {$b$};
\node[dot=9pt,draw=black,fill=myred] at (C2) {$c$};
\node[dot=9pt,draw=black,fill=mygreen] at (A3) {$a$};
\node[dot=9pt,draw=black,fill=mygreen] at (B3) {$b$};
\node[dot=9pt,draw=black,fill=mygreen] at (C3) {$c$};
\node[dot=9pt,draw=black,fill=myblue] at (A4) {$a$};
\node[dot=9pt,draw=black,fill=mygreen] at (B4) {$b$};
\node[dot=9pt,draw=black,fill=mygreen] at (C4) {$c$};
\end{tikzpicture}
\caption{$\Kone{\TG}$ constructed for $\TG$ from \Cref{fig:snapshot}}
\label{fig:Kone}
\end{figure}

\begin{definition}
Let ${\TG = (\G_1,t_1), \dots, (\G_n,t_n)}$ be a temporal graph with   $\G_i=(V_i,E_i,c_i)$.
We define a  
knowledge graph $\Kone{\TG} = (V,E,R,c)$ with components:
\begin{itemize}
\item $V = \TV(\TG)$, 
\item $E= \{ (t_j-t_i,(v,t_i),(u,t_j)) \mid i\leq j \text{ and } \{u,v\} \in E_i \}$, 
\item $R=\{0,\dots,n -1 \}$, 
\item $c:V \to R$ satisfies\footnote{for brevity we will drop double brackets, e.g. from $c((v,t_i))$} $c(v,t_i) = c_i(v)$, for all $(v,t_i) \in V$.
\end{itemize}
\end{definition}

We use $\Kone{\TG}$ to bridge the expressive power of global \MPTGNN{}s and 1-WL. 
First we  show that global \MPTGNN{}s cannot distinguish more timestamped nodes over \TG{} than 1-WL over $\Kone{\TG}$.

\begin{restatable}{thm}{Konei}\label{thm:Kone1}
For any  temporal graph $\TG$, any timestamped nodes $(v,t)$ and $(u,t')$ in $\TG$, and any $\ell \in \N$:
\begin{itemize}
\item 
If 
$\rwl^{(\ell)}(v,t) = \rwl^{(\ell)}(u,t')$ in $\Kone{\TG}$,
\item 
then $\h_v^{(\ell)}(t)=\h_u^{(\ell)}(t')$ in any global \TGNN{}. 
\end{itemize}
\end{restatable}
\begin{proof}[Proof sketch]
By induction on $\ell$.
The base case holds since $\rwl^{(0)}(v,t_i)=c(v,t_i)=c_i(v)=\h_v^{(0)}(t_i)$, for each $(v,t_i)$ in $\TG$.
The inductive step proceeds in a similar way to \GNN{}s~\cite{DBLP:conf/aaai/0001RFHLRG19},
but additionally exploits
the following key property of $\Kone{\TG}$: 
$(u,t') \in \NH_r(v,t)$ in $\Kone{\TG}$ iff  
$r=t-t'$ and $(u,t') \in \NH(v,t)$ in $\TG$,
for all timestamped nodes $(v,t)$ and $(u,t')$  in $\TG$.
\end{proof}

Moreover, we can show the opposite direction: there is a global \MPTGNN{} distinguishing exactly the same nodes over $\TG$ as 1-WL over $\Kone{\TG}$.
By \Cref{thm:Kone1} each global \MPTGNN{} is not more expressive than 1-WL over $\Kone{\TG}$, so for the next theorem it suffices to construct an \MPTGNN{} at least as expressive as 1-WL over $\Kone{\TG}$.

\begin{restatable}{thm}{Koneii}\label{thm:Kone2}
For any temporal graph $\TG$ and any $L \in \N$, there exists a global 
\MPTGNN{} $\model$ with $L$ layers 
such that for all timestamped nodes $(v,t), (u,t')$ in $\TG$ and all $\ell \leq L$ the following are equivalent:
\begin{itemize}
\item 
$\h_v^{(\ell)}(t)=\h_u^{(\ell)}(t')$ in $\model$,
\item 
$\rwl^{(\ell)}(v,t) = \rwl^{(\ell)}(u,t')$ in $\Kone{\TG}$.
\end{itemize}
\end{restatable}
\begin{proof}[Proof sketch]
The important part of the proof is for the forward implication, as the other implication follows from \Cref{thm:Kone1}.
We use the result of \citet{DBLP:conf/nips/Huang0CB23}[Theorem A.1] showing that for any knowledge graph and in particular $\Kone{\TG}$,
there is a relational message-passing neural network (R-MPNN) model $\modelB$ such that if two nodes $(v,t)$ and $(u,t')$ of $\Kone{\TG}$ have the  same embeddings  at a layer $\ell$, we get 
$\rwl^{(\ell)}(v,t) = \rwl^{(\ell)}(u,t')$.
\citeauthor{DBLP:conf/nips/Huang0CB23}'s model $\modelB$ 
computes $\hh_{(v,t)}^{(\ell)}$ as follows:
$
\hh_{(v,t)}^{(0)}   =  c(v,t)$ and
$\hh_{(v,t)}^{(\ell)} = \sign \Big( 
\W^{(\ell)} (\hh_{(v,t)}^{(\ell-1)}  +  
 \sum_{r\in R} \; \; \sum_{(u,t') \in \NH_{r}(v,t)} 
\alpha_r \hh_{(u,t')}^{(\ell-1)} )  -  \bb \Big).
$
To finish the proof, we construct a global \MPTGNN{} such that
$\h^{(\ell)}_v(t)$
computed by $\model$  on $\TG$  coincide with  $\hh_{(v,t)}^{(\ell)}$ computed by $\modelB$ on $\Kone{\TG}$.
We obtain it by setting  in Equation~\eqref{TGN2} functions
$\agg^{(\ell)}$ to the sum and
$\com^{(\ell)}$ to the sign of a particular linear combination.
\end{proof}

Next, we show
that we can also construct a knowledge graph representing message-passing in local \MPTGNN{}s.
In contrast to $\Kone{\TG}$,  edges 
of the new knowledge graph $\Ktwo{\TG}$
are bidirectional and hold only between nodes stamped with the same time.
Such a knowledge graph is presented in \Cref{fig:Ktwo} and  formally defined below.

\begin{figure}[ht]
\centering
\begin{tikzpicture}[
dot/.style = {draw, circle, minimum size=#1,
              inner sep=0pt, outer sep=0pt},
dot/.default = 6pt
]
\scriptsize
\pgfmathsetmacro{\tline}{-0.25}
\pgfmathsetmacro{\w}{1.3}
\pgfmathsetmacro{\h}{1.5}
\pgfmathsetmacro{\inh}{0.7}
\pgfmathsetmacro{\dist}{1.9}
\pgfmathsetmacro{\Ax}{0.5}
\pgfmathsetmacro{\Ay}{1.4}
\pgfmathsetmacro{\Bx}{0.4}
\pgfmathsetmacro{\By}{0.5}
\pgfmathsetmacro{\Cx}{1}
\pgfmathsetmacro{\Cy}{1}

\draw[->] (1.8,\tline) -- (9.1,\tline);

\foreach \x in {1,...,4}
{
\draw[fill=gray!90!black,opacity=0.2] (\x*\dist,0) -- (\x*\dist+\w,\inh) -- (\x*\dist+\w,\inh+\h) -- (\x*\dist,\h) -- cycle;

\node at (\x*\dist + 0.87 * \w, \inh+ \h - 0.3) {$G_{\x}$};
\node at (\x*\dist + 0.6 * \w, 0) {$t_{\x}={\x}$};
\draw[-] (\x*\dist + 0.6 * \w, \tline+0.05) -- (\x*\dist + 0.6 * \w, \tline-0.05);

\node[dot=9pt,draw=none] (A\x) at (\x*\dist+\Ax, 0+\Ay) {};
\node[dot=9pt,draw=none] (B\x) at (\x*\dist+\Bx, 0+\By) {};
\node[dot=9pt,draw=none] (C\x) at (\x*\dist+\Cx, 0+\Cy) {};
}

\draw[<->,blue,thick] (A2) -- (B2) node[midway,left] {0};
\draw[<->,blue,thick] (A3) -- (B3) node[midway,left] {1};
\draw[<->,blue,thick] (B3) -- (C3) node[pos=0.3,above] {0};
\draw[<->,blue,thick] (A4) -- (B4) node[midway,left] {2};
\draw[<->,blue,thick] (A4) -- (C4) node[pos=-0.08,right=0.07] {0};
\draw[<->,blue,thick] (B4) -- (C4) node[pos=0.3,above] {1};
\draw (B4) edge[<->,blue,thick,bend right=20] node[pos=0.7,below] {0} (C4);

\node[dot=9pt,draw=black,fill=myblue] at (A1) {$a$};
\node[dot=9pt,draw=black,fill=mygreen] at (B1) {$b$};
\node[dot=9pt,draw=black,fill=myred] at (C1) {$c$};
\node[dot=9pt,draw=black,fill=mygreen] at (A2) {$a$};
\node[dot=9pt,draw=black,fill=mygreen] at (B2) {$b$};
\node[dot=9pt,draw=black,fill=myred] at (C2) {$c$};
\node[dot=9pt,draw=black,fill=mygreen] at (A3) {$a$};
\node[dot=9pt,draw=black,fill=mygreen] at (B3) {$b$};
\node[dot=9pt,draw=black,fill=mygreen] at (C3) {$c$};
\node[dot=9pt,draw=black,fill=myblue] at (A4) {$a$};
\node[dot=9pt,draw=black,fill=mygreen] at (B4) {$b$};
\node[dot=9pt,draw=black,fill=mygreen] at (C4) {$c$};
\end{tikzpicture}
\caption{Knowledge graph $\Ktwo{\TG}$ for $\TG$ from \Cref{fig:snapshot}}
\label{fig:Ktwo}
\end{figure}
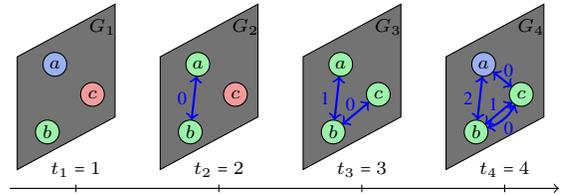

\begin{definition}
Let ${\TG = (\G_1,t_1), \dots, (\G_n,t_n)}$ be a temporal graph with   $\G_i=(V_i,E_i,c_i)$.
We define a  knowledge 
knowledge graph $\Ktwo{\TG} = (V,E,R,c)$ with:
\begin{itemize}
\item $V = \TV(\TG)$, 
\item $E= \{ (t_j-t_i,(v,t_j),(u,t_j)) \mid i\leq j \text{ and } \{u,v\} \in E_i \}$, 
\item $R=\{0,\dots,n-1 \}$, 
\item $c:V \to R$ satisfies $c(v,t_i) = c_i(v)$, for all $(v,t_i) \in V$.
\end{itemize}
\end{definition}

We can show that local \MPTGNN{}s can distinguish exactly the same nodes in $\TG$ as $\rwl$ can distinguish in $\Ktwo{\TG}$, as formally stated in the following two theorems.

\begin{restatable}{thm}{Ktwoi}\label{thm:Ktwo1}
For any temporal graph $\TG$, any timestamped nodes $(v,t)$ and $(u,t')$ in $\TG$, and any $\ell \in \N$:
\begin{itemize}
\item 
If 
$\rwl^{(\ell)}(v,t) = \rwl^{(\ell)}(u,t')$ in $\Ktwo{\TG}$,
\item 
then $\h_v^{(\ell)}(t)=\h_u^{(\ell)}(t')$ in any local \MPTGNN{}.
\end{itemize}
\end{restatable}

\begin{restatable}{thm}{Ktwoii}\label{thm:Ktwo2}
For any  temporal graph $\TG$ and any $L \in \N$, there exists a local \MPTGNN{} $\model$ with $L$ layers 
such that for all timestamped nodes $(v,t), (u,t')$ in $\TG$ and all $\ell \leq L$, the following are equivalent:
\begin{itemize}
\item 
$\h_v^{(\ell)}(t)=\h_u^{(\ell)}(t')$ in $\model$,
\item 
$\rwl^{(\ell)}(v,t) = \rwl^{(\ell)}(u,t')$ in $\Ktwo{\TG}$.
\end{itemize}
\end{restatable}

The
Weisfeiler-Leman characterisation 
of global and local \MPTGNN{}s established in the above theorems provides us with a versatile tool for analysing expressive power, which we will intensively apply in the following parts of the paper.

\section{Timewise Isomorphism}

While message-passing \GNN{}s (corresponding to 1-WL) provide us with a heuristic for graph isomorphism, their temporal extensions can be seen as heuristics for isomorphism between temporal graphs. 
However,  in the temporal setting it is not clear what notion of  isomorphism we should use to obtain an analogous correspondence.
We use the characterisation from the previous section to show, quite surprisingly, 
that both global and local \MPTGNN{}s
can distinguish nodes which are \emph{pointwise isomorphic}---called isomorphic by~\citet{DBLP:journals/nn/BeddarWiesingDGLMST24}.
This observation leads us to definition of \emph{timewise isomorphism} as a suitable notion for node indistinguishability in temporal graphs.

Pointwise isomorphism  requires that pairs of corresponding snapshots in two temporal graphs are isomorphic. 
For example $(a,t_2)$ in $\TG$ and $(a',t_2)$ in $\TG'$ from \Cref{fig:point} are pointwise isomorphic since  
$f_1$  with 
$f_1(a)=b'$, $f_1(b)=c'$, and $f_1(c)=a'$  is an isomorphism between $G_1$ and $G_1'$, and   $f_2$ with 
$f_2(a)=a'$, $f_2(b)=b'$, and $f_2(c)=c'$
is an isomorphism between $G_2$ and $G_2'$.
A formal definition is below.

\begin{figure}[ht]
\centering
\begin{tikzpicture}[
dot/.style = {draw, circle, minimum size=#1,
              inner sep=0pt, outer sep=0pt},
dot/.default = 6pt
]
\scriptsize
\pgfmathsetmacro{\tline}{-0.25}
\pgfmathsetmacro{\w}{1.3}
\pgfmathsetmacro{\h}{1.5}
\pgfmathsetmacro{\inh}{0.7}
\pgfmathsetmacro{\dist}{1.7}
\pgfmathsetmacro{\Ax}{0.5}
\pgfmathsetmacro{\Ay}{1.4}
\pgfmathsetmacro{\Bx}{0.4}
\pgfmathsetmacro{\By}{0.5}
\pgfmathsetmacro{\Cx}{1}
\pgfmathsetmacro{\Cy}{1}

\draw[->] (1.8,\tline) -- (4.8,\tline);

\foreach \x in {1,...,2}
{
\draw[fill=gray!90!black,opacity=0.2] (\x*\dist,0) -- (\x*\dist+\w,\inh) -- (\x*\dist+\w,\inh+\h) -- (\x*\dist,\h) -- cycle;

\node at (\x*\dist + 0.87 * \w, \inh+ \h - 0.3) {$G_{\x}$};
\node at (\x*\dist + 0.6 * \w, 0) {$t_{\x}={\x}$};
\draw[-] (\x*\dist + 0.6 * \w, \tline+0.05) -- (\x*\dist + 0.6 * \w, \tline-0.05);

\node (A\x) at (\x*\dist+\Ax, 0+\Ay) {};
\node (B\x) at (\x*\dist+\Bx, 0+\By) {};
\node (C\x) at (\x*\dist+\Cx, 0+\Cy) {};
}

\draw[thick] (A1) -- (B1);

\node[dot=11pt,draw=black,fill=mygreen] at (A1) {$a$};
\node[dot=11pt,draw=black,fill=mygreen] at (B1) {$b$};
\node[dot=11pt,draw=black,fill=mygreen] at (C1) {$c$};
\node[dot=11pt,draw=black,fill=mygreen] at (A2) {$a$};
\node[dot=11pt,draw=black,fill=mygreen] at (B2) {$b$};
\node[dot=11pt,draw=black,fill=mygreen] at (C2) {$c$};
\node at (3.6,2.2) {$\TG$};
\end{tikzpicture}
\qquad
\begin{tikzpicture}[
dot/.style = {draw, circle, minimum size=#1,
              inner sep=0pt, outer sep=0pt},
dot/.default = 6pt
]
\scriptsize
\pgfmathsetmacro{\tline}{-0.25}
\pgfmathsetmacro{\w}{1.3}
\pgfmathsetmacro{\h}{1.5}
\pgfmathsetmacro{\inh}{0.7}
\pgfmathsetmacro{\dist}{1.7}
\pgfmathsetmacro{\Ax}{0.5}
\pgfmathsetmacro{\Ay}{1.4}
\pgfmathsetmacro{\Bx}{0.4}
\pgfmathsetmacro{\By}{0.5}
\pgfmathsetmacro{\Cx}{1}
\pgfmathsetmacro{\Cy}{1}

\draw[->] (1.8,\tline) -- (4.8,\tline);

\foreach \x in {1,...,2}
{
\draw[fill=gray!90!black,opacity=0.2] (\x*\dist,0) -- (\x*\dist+\w,\inh) -- (\x*\dist+\w,\inh+\h) -- (\x*\dist,\h) -- cycle;

\node at (\x*\dist + 0.87 * \w, \inh+ \h - 0.3) {$G'_{\x}$};
\node at (\x*\dist + 0.6 * \w, 0) {$t_{\x}={\x}$};
\draw[-] (\x*\dist + 0.6 * \w, \tline+0.05) -- (\x*\dist + 0.6 * \w, \tline-0.05);

\node (A\x) at (\x*\dist+\Ax, 0+\Ay) {};
\node (B\x) at (\x*\dist+\Bx, 0+\By) {};
\node (C\x) at (\x*\dist+\Cx, 0+\Cy) {};
}

\draw[thick] (B1) -- (C1);

\node[dot=11pt,draw=black,fill=mygreen] at (A1) {$a'$};
\node[dot=11pt,draw=black,fill=mygreen] at (B1) {$b'$};
\node[dot=11pt,draw=black,fill=mygreen] at (C1) {$c'$};
\node[dot=11pt,draw=black,fill=mygreen] at (A2) {$a'$};
\node[dot=11pt,draw=black,fill=mygreen] at (B2) {$b'$};
\node[dot=11pt,draw=black,fill=mygreen] at (C2) {$c'$};
\node at (3.6,2.2) {$\TG'$};
\end{tikzpicture}
\caption{Pointwise isomorphic $(a,t_2)$ and $(a',t_2)$
}
\label{fig:point}
\end{figure}
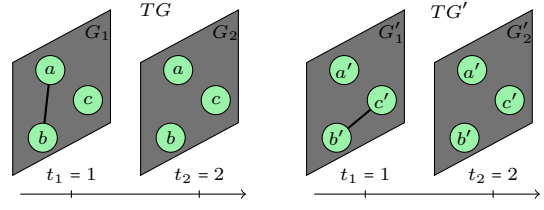

\begin{definition}
Temporal graphs $\TG\!=\! (\G_1,t_1), \dots, (\G_n,t_n)$ and $\TG' = (\G_1',t_1'), \dots, (\G_m',t_m')$ are
\emph{pointwise isomorphic}
if both of the following hold:
\begin{itemize}
\item $\T(\TG)=\T(\TG')$ (so $n=m$ and $t_i=t'_i$ for all $i \in \{1, \dots, n \}$)
\item for every $i \in \{ 1, \dots , n \}$ there exists an isomorphism $f_i$
between $\G_i$ and $\G_i'$.
\end{itemize}
If this is the case and $f_i(v) = u$, we say that   $(v,t_i)$ and $(u,t_i)$ are \emph{pointwise isomorphic}.
\end{definition}

It turns out that both global and local \MPTGNN{}s can distinguish 
pointwise isomorphic nodes.
In particular, they can distinguish
$(a,t_2)$ and $(a',t_2)$ from \Cref{fig:point}, as we show below using
\Cref{thm:Kone2} and \Cref{thm:Ktwo2}. 

\begin{restatable}{thm}{piso}\label{thm:piso}
There are temporal graphs $\TG$ and $\TG'$ with pointwise isomorphic $(v,t)$ and $(u,t')$ such that  
$\h_{v}^{(1)}(t) \neq \h_{u}^{(1)}(t')$ for some  
global and local \MPTGNN{}s.
\end{restatable}
\begin{proof}[Proof sketch]
Consider $\TG$ and $\TG'$ from \Cref{fig:point}, where $(a,t_2)$ is pointwise isomorphic to $(a',t_2)$.
If we apply $\rwl{}$ to $\Kone{\TG}$ and $\Kone{\TG'}$, $\rwl^{(1)}(a,t_2) \neq \rwl^{(1)}(a',t_2)$, because $(a,t_2)$ has one incoming edge in $\Kone{\TG}$, but $(a',t_2)$ has no incoming edges in $\Kone{\TG'}$.
The same holds if we apply $\rwl$ to $\Kone{\TG'}$  and $\Ktwo{\TG'}$.
So, by \Cref{thm:Kone2} and \Cref{thm:Ktwo2}, there are global and local \MPTGNN{}s in which $\h_{a}^{(1)}(t_2) \neq \h_{a'}^{(1)}(t_2)$.
\end{proof}

\Cref{thm:piso} shows that pointwise isomorphism is unsuitable for detecting node indistinguishability in  \MPTGNN{}s.
We obtain an adequate isomorphism notion by, on the one hand, requiring additionally that all $f_i$ mentioned in the definition of  pointwise isomorphism coincide but, on the other hand, relaxing the requirement $\T(\TG)=\T(\TG')$.

\begin{definition}\label{def:iso}
Temporal graphs $\TG\!=\! (\G_1,t_1), \dots, (\G_n,t_n)$ and $\TG' = (\G_1',t_1'), \dots, (\G_m',t_m')$ are
\emph{timewise isomorphic}
if both of the following hold:
\begin{itemize}
\item $n=m$ and $t_{i+1} - t_i = t'_{i+1} - t'_i$, for  every $i \in \{1, \dots, n-1 \}$,
\item  there exists a function  $f$ which is an isomorphism
between $\G_i$ and $\G_i'$,  for every $i \in \{ 1, \dots , n \}$.
\end{itemize}
If $f(v) = u$, we say that   $(v,t_i)$ and $(u,t'_i)$ are \emph{timewise isomorphic}, for any $t_i \in \T(\TG)$.
\end{definition}

Next we show that
the timewise isomorphism is an adequate notion for timestamped nodes indistinguishability since timestamped nodes which are timewise isomorphic  cannot be distinguished by any (global or local) \MPTGNN{}.

\begin{restatable}{thm}{iso}\label{thm:iso}
If $(v,t)$ and $(u,t')$ are timewise isomorphic, then
$\h_{v}^{(\ell)}(t) = \h_{u}^{(\ell)}(t')$ in any  
\MPTGNN{} and any $\ell \in \N$.
\end{restatable}
\begin{proof}[Proof sketch]
Assume that  $(v,t)$ from $\TG$ and  $(u,t')$ from $\TG'$ are  timewise isomorphic.
Hence, by \Cref{def:iso}, $\TG$ and $\TG'$ are of the forms $\TG = (\G_1,t_1), \dots, (\G_n,t_n)$ and $\TG' = (\G'_1,t'_1), \dots, (\G'_n,t'_n)$, as well as  $t=t_i$ and $t'=t'_i$ for some $i \in \{ 1, \dots, n\}$.
Moreover, $f(v) = u$ for some ${f:V(\TG) \to V(\TG')}$ satisfying requirements in \Cref{def:iso}.
We define $f': \TV(\TG) \to \TV(\TG')$ such that $f'(w,t_j)=(f(w),t'_j)$  for  all 
$(w,t_j) \in \TV(\TG)$.
We can show that $f'$ is an isomorphism between knowledge graphs $\Kone{\TG}$ and $\Kone{\TG'}$,
as well as between $\Ktwo{\TG}$ and $\Ktwo{\TG'}$.
Hence, in both cases $f'(v,t) = (u,t')$ implies
$\rwl^{(\ell)}(v,t)=\rwl^{(\ell)}(u,t')$, for all $\ell \in \N$.
Thus, by \Cref{thm:Kone1} and \Cref{thm:Ktwo1}, we obtain that $\h_{v}^{(\ell)}(t) = \h_{u}^{(\ell)}(t')$ for any global and local \MPTGNN{}s.
\end{proof}

\section{Relative Expressiveness of Temporal Message Passing Mechanisms}

In this section we will use temporal Weisfeiler-Leman characterisation  to prove expressive power results summarised in \Cref{fig:express}.
Our results are on the \emph{discriminative} (also called  separating)  power, which aims to determine if a given type of models is able to distinguish two timestamped nodes. 
Formally, we say that a model \emph{distinguishes} a timestamped node  $(v,t)$ in a temporal graph $\TG$ from  $(u,t')$ in $\TG'$ if this model computes different embedding for $(v,t)$ and $(u,t')$ at some layer $\ell$, that is,
 $\h_{v}^{(\ell)}(t) \neq \h_{u}^{(\ell)}(t')$.
We say that a type of models (e.g. global or local \MPTGNN{}s) \emph{can distinguish} $(v,t)$ from 
$(u,t')$, 
if some model of this type  distinguishes $(v,t)$ from $(u,t')$.

We start by showing that global \MPTGNN{}s can distinguish timestamped nodes which are indistinguishable by local \MPTGNN{}s. 
The reason is that in a global \MPTGNN{} an embedding of $(v,t)$  can depend on embeddings 
at $t'<t$, but this cannot happen in a local \MPTGNN{}.

\begin{restatable}{thm}{globmoreloc}\label{thm:globmoreloc}
There are timestamped nodes  that can be distinguished by global, but not by local \MPTGNN{}s.
\end{restatable}
\begin{proof}[Proof sketch]
Consider $(b,t_4)$ from $\TG$ in
\Cref{fig:snapshot} and $(b,t_4)$ from $\TG'$ in \Cref{fig:aggregated}~(\subref{fig:snap}).
We can show that, for any $\ell \in \N$, application of $\rwl^{(\ell)}$ to $\Ktwo{\TG}$ and $\Ktwo{\TG'}$ assigns the same labels to these timestamped nodes.  
Hence, by \Cref{thm:Ktwo1}, local \MPTGNN{}s cannot distinguish these nodes.
On the other hand, for any $\ell \geq 1$,  application of $\rwl^{(\ell)}$ to $\Kone{\TG}$ and $\Kone{\TG'}$ assigns different labels to these nodes.
Therefore, by \Cref{thm:Kone2}, global \MPTGNN{}s can distinguish these nodes.
\end{proof}

Based on the observation from \Cref{thm:globmoreloc}, one could expect that global \MPTGNN{}s are strictly more expressive than local \MPTGNN{}.
Surprisingly, this is not the case.
Indeed, as we show next, there are timestamped nodes which
can be distinguished by local, but not by global \MPTGNN{}s.

\begin{restatable}{thm}{locmoreglob}\label{thm:locmoreglob}
There are 
timestamped nodes that can be distinguished by local, but not by global \MPTGNN{}s.
This holds true even for colour-persistent temporal graphs.
\end{restatable}
\begin{proof}[Proof sketch]
Consider $(a,t_2)$ from $\TG$ and $(a',t_2)$ from $\TG'$  in \Cref{fig:Ttwomore}.
Observe that $(a,t_2)$ in $\Kone{\TG}$ is  isomorphic to $(a',t_2)$ in $\Kone{\TG'}$ so, by \Cref{thm:Kone1}, $(a,t_2)$ and $(a',t_2)$ cannot be distinguished by global \MPTGNN{}s.
However, $(a,t_2)$  has one outgoing path of length 2 in $\Ktwo{\TG}$, but not in  $\Ktwo{\TG'}$.
Hence, two iterations of $\rwl$  distinguish these nodes.
Thus, by \Cref{thm:Ktwo2}, $(a,t_2)$ and $(a',t_2)$ can be distinguished by local \MPTGNN{}s.
Note that $\TG$ and $\TG'$ are colour-persistent.
\end{proof}
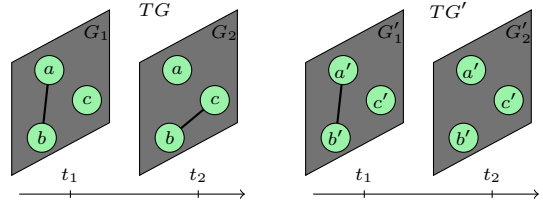
\begin{figure}[ht]
\centering
\begin{tikzpicture}[
dot/.style = {draw, circle, minimum size=#1,
              inner sep=0pt, outer sep=0pt},
dot/.default = 6pt
]
\scriptsize
\pgfmathsetmacro{\tline}{-0.25}
\pgfmathsetmacro{\w}{1.3}
\pgfmathsetmacro{\h}{1.5}
\pgfmathsetmacro{\inh}{0.7}
\pgfmathsetmacro{\dist}{1.7}
\pgfmathsetmacro{\Ax}{0.5}
\pgfmathsetmacro{\Ay}{1.4}
\pgfmathsetmacro{\Bx}{0.4}
\pgfmathsetmacro{\By}{0.5}
\pgfmathsetmacro{\Cx}{1}
\pgfmathsetmacro{\Cy}{1}

\draw[->] (1.8,\tline) -- (4.8,\tline);

\foreach \x in {1,...,2}
{
\draw[fill=gray!90!black,opacity=0.2] (\x*\dist,0) -- (\x*\dist+\w,\inh) -- (\x*\dist+\w,\inh+\h) -- (\x*\dist,\h) -- cycle;

\node at (\x*\dist + 0.87 * \w, \inh+ \h - 0.3) {$G_{\x}$};
\node at (\x*\dist + 0.6 * \w, 0) {$t_{\x}$};
\draw[-] (\x*\dist + 0.6 * \w, \tline+0.05) -- (\x*\dist + 0.6 * \w, \tline-0.05);

\node (A\x) at (\x*\dist+\Ax, 0+\Ay) {};
\node (B\x) at (\x*\dist+\Bx, 0+\By) {};
\node (C\x) at (\x*\dist+\Cx, 0+\Cy) {};
}

\draw[thick] (A1) -- (B1);
\draw[thick] (B2) -- (C2);

\node[dot=11pt,draw=black,fill=mygreen] at (A1) {$a$};
\node[dot=11pt,draw=black,fill=mygreen] at (B1) {$b$};
\node[dot=11pt,draw=black,fill=mygreen] at (C1) {$c$};
\node[dot=11pt,draw=black,fill=mygreen] at (A2) {$a$};
\node[dot=11pt,draw=black,fill=mygreen] at (B2) {$b$};
\node[dot=11pt,draw=black,fill=mygreen] at (C2) {$c$};
\node at (3.6,2.2) {$\TG$};
\end{tikzpicture}
\qquad
\begin{tikzpicture}[
dot/.style = {draw, circle, minimum size=#1,
              inner sep=0pt, outer sep=0pt},
dot/.default = 6pt
]
\scriptsize
\pgfmathsetmacro{\tline}{-0.25}
\pgfmathsetmacro{\w}{1.3}
\pgfmathsetmacro{\h}{1.5}
\pgfmathsetmacro{\inh}{0.7}
\pgfmathsetmacro{\dist}{1.7}
\pgfmathsetmacro{\Ax}{0.5}
\pgfmathsetmacro{\Ay}{1.4}
\pgfmathsetmacro{\Bx}{0.4}
\pgfmathsetmacro{\By}{0.5}
\pgfmathsetmacro{\Cx}{1}
\pgfmathsetmacro{\Cy}{1}

\draw[->] (1.8,\tline) -- (4.8,\tline);

\foreach \x in {1,...,2}
{
\draw[fill=gray!90!black,opacity=0.2] (\x*\dist,0) -- (\x*\dist+\w,\inh) -- (\x*\dist+\w,\inh+\h) -- (\x*\dist,\h) -- cycle;

\node at (\x*\dist + 0.87 * \w, \inh+ \h - 0.3) {$G'_{\x}$};
\node at (\x*\dist + 0.6 * \w, 0) {$t_{\x}$};
\draw[-] (\x*\dist + 0.6 * \w, \tline+0.05) -- (\x*\dist + 0.6 * \w, \tline-0.05);

\node (A\x) at (\x*\dist+\Ax, 0+\Ay) {};
\node (B\x) at (\x*\dist+\Bx, 0+\By) {};
\node (C\x) at (\x*\dist+\Cx, 0+\Cy) {};
}

\draw[thick] (A1) -- (B1);

\node[dot=11pt,draw=black,fill=mygreen] at (A1) {$a'$};
\node[dot=11pt,draw=black,fill=mygreen] at (B1) {$b'$};
\node[dot=11pt,draw=black,fill=mygreen] at (C1) {$c'$};
\node[dot=11pt,draw=black,fill=mygreen] at (A2) {$a'$};
\node[dot=11pt,draw=black,fill=mygreen] at (B2) {$b'$};
\node[dot=11pt,draw=black,fill=mygreen] at (C2) {$c'$};
\node at (3.6,2.2) {$\TG'$};
\end{tikzpicture}
\caption{$(a,t_2)$ and $(a',t_2)$ which cannot be distinguished by global, but can be distinguished by local \MPTGNN{}s  
}
\label{fig:Ttwomore}
\end{figure}

\Cref{thm:globmoreloc} and \Cref{thm:locmoreglob} show us that neither global or local \MPTGNN{}s are strictly more expressive, when compared over all temporal graphs.
Does the same result hold over colour-persistent  graphs?
Interestingly, it is not the case: in colour-persistent graphs
local \MPTGNN{}s are strictly more expressive than global \MPTGNN{}s.
Hence \Cref{thm:globmoreloc} cannot hold for colour-persistent graphs.




\begin{restatable}{thm}{more}\label{thm:more}
In colour-persistent graphs local \MPTGNN{}s are strictly more expressive than global \MPTGNN{}s. 
\end{restatable}
\begin{proof}[Proof sketch]
Due to the result established in \Cref{thm:locmoreglob}, it remains to show that 
over colour-persistent temporal graphs, if $(v,t)$ and $(u,t')$  can be distinguished by global \MPTGNN{}s, then they can be distinguished also by local \MPTGNN{}s.
Hence, by \Cref{thm:Kone1} and \Cref{thm:Ktwo2},
we need to show the 
$\rwlg^{(\ell)}(v,t) \neq \rwlg^{(\ell)}(u,t')$   implies $\rwll^{(\ell)}(v,t) \neq \rwll^{(\ell)}(u,t')$ for all $\ell \in \N$.
We show this implication inductively on $\ell$, where the inductive step requires proving several non-trivial statements, for example, showing (by another induction) that $\rwl^{(\ell)}(v,t) \neq \rwl^{(\ell)}(u,t')$ implies  $\rwl^{(\ell)}(v,t+k) \neq \rwl^{(\ell)}(u,t'+k)$, for any $k$.
\end{proof}




To finish the expressive power landscape announced in \Cref{fig:express}, it remains to make two more observations.
On the one hand, temporal graphs which are not colour-persistent allow global \MPTGNN{}s to distinguish more elements than colour-persistent graphs. 
Indeed, this is the case since global \MPTGNN{}s allow us to pass information about colours between nodes stamped with different time points.
On the other hand, this is not allowed in local \MPTGNN{}s, and so colour-persistence does not impact their expressiveness.





\section{Experiments}









We implement and train basic variants of global and local models
on standard temporal link-prediction benchmarks.
We emphasise that the goal of our experiments is not to achieve models with high-level performance, but 
to examine how our expressive power results impact practical performance of \MPTGNN{}s.


\paragraph{Benchmarks.}
We use the Temporal Graph Benchmark (TGB) 2.0 suite~\cite{gastinger2024tgb} with small-to-medium temporal datasets \texttt{tgbl-wiki}, \texttt{tgbl-review}, and \texttt{tgbl-coin}, whose statistics are in \Cref{tab:results}.
They do not have node features and we discard the edge features.
We consider a link-prediction task, where the goal is to predict whether there is a link between two given nodes at the next time point, given information about all previous links.
We follow normative training and evaluation procedures supplied by TGB.

\paragraph{Models.}
We implement global and local \MPTGNN{}s with
combination and aggregation functions being concatenation ($\con$) and summation ($\sum$), respectively, which are among  standard choices \cite{DBLP:conf/aaai/RossiA15,DBLP:conf/iclr/XuHLJ19}.
Hence, embedding $\h_v^{(\ell)}(t)$ is computed as
$$
W_2^{(\ell)}
[
\h_v^{(\ell-1)}(t) \con
\sigma
(
W_1^{(\ell)}
(
\sum_{\substack{(u, t') \in \NH(v,t)}}
\fone \con g(t-t')
) 
)
],
$$
where $W_1$ and $W_2$ are learnable,  $\sigma$ is the rectified linear unit,  $\fone$
is either $\mathbf{h}_u^{(\ell-1)}(t')$ (giving rise to a global model) or $\mathbf{h}_u^{(\ell-1)}(t)$ (local model), and $g(t - t') = t - t'$.
After $\ell = 4$  layers, a multi-layer perceptron with 1024 hidden units predicts a link between $u$ and $v$, given $\h_u^{(\ell)}(t)$ and $\h_v^{(\ell)}(t)$.
Aggregating the entire temporal neighbourhood incurs over time a linear computational penalty, precluding larger benchmarks; real-world models approximate this calculation.

\paragraph{Implementation.}
Our implementation\footnote{Implementation is included in supplementary materials.} is based on PyTorch \cite{DBLP:conf/nips/PaszkeGMLBCKLGA19},
and in particular its hardware-accelerated scatter operations for temporal aggregation. The use of scatter operations means that results may differ between hardware-accelerated runs.

As stated, the global model would require enormous compute and memory in order to use node embeddings from all previous time points.
In order to make this tractable, we apply a train-time approximation: during an epoch, embeddings from previous timepoints are ``frozen'': detached from the computation graph and not updated as model weights change.
This interferes with training as the model must use a mixture of stale and fresh embeddings, but at test time the result is exact.
A further observation is that only those embeddings whose nodes are connected at some time need be computed and retained due to the definition of $\NH$.

Minibatching can be achieved in the temporal context by predicting the next $k$ links given all previous links.
Unlike traditional minibatching, this can have a detrimental impact on model accuracy, because earlier links in the batch may help to predict links later in the batch.
However, it is computationally very demanding to set $k = 1$ for even ``small'' datasets like \texttt{tgbl-review} containing millions of links, so a compromise must be found. We set $k = 32$ for \texttt{tgbl-wiki} and $k = 1024$ for all others.

\paragraph{Training}
We used Adam~\cite{DBLP:journals/corr/KingmaB14} for optimisation with PyTorch defaults $\gamma = 0.001$, $\beta_1 = 0.9$, $\beta_2 = 0.999$, and no L2 penalty.
We were able to significantly stabilise and accelerate training by normalising $g(t - t')$ with respect to elapsed time and by applying batch normalisation~\cite{DBLP:conf/icml/IoffeS15} immediately after summation of temporal neighbours.
With the exception of the above stability measures, we have not tuned further as we are not aiming for state-of-the-art performance.
Training continued until validation loss failed to improve for 10 epochs.
Experiments involving \texttt{tgbl-wiki} and \texttt{tgbl-review} can be run on   desktop hardware (NVIDIA GT730), or even without acceleration, whereas \texttt{tgbl-coin} requires  a large GPU.

\begin{table}
    \caption{Statistics (nodes and edges) and MRR scores}
    \label{tab:results}
    \centering
    \setlength{\tabcolsep}{3pt}
    \begin{tabular}{l || r | r | r |}
                & \texttt{tgbl-wiki} & \texttt{tgbl-review} & \texttt{tgbl-coin} \\
      \hline \hline 
        nodes & 9,227  & 352,637 & 638,486 \\
        edges & 157,474 & 4,873,540 & 22,809,486 \\                
      \hline 
        global & 0.223 & 0.321 & 0.628 \\
        local & \textbf{0.264} & \textbf{0.359} & \textbf{0.635} \\
        \hline
    \end{tabular}
\end{table}

\paragraph{Results.}
\Cref{tab:results} shows the
mean reciprocal rank (MRR) score (higher is better) used in 
TGB. 
We observe that the scores are relatively high given the simplicity of models and lack of tuning.
We have written in bold higher among MRRs obtained by global and local \MPTGNN{}s.
In all three datasets, 
local \MPTGNN{} obtains higher scores, but the difference between the scores of local and global models is relatively small.
We observe that 
higher performance of the local model aligns with our theoretical result from \Cref{thm:more}, which states that over colour-persistent temporal graphs (as here), local \MPTGNN{}s are stricly more expressive than global \MPTGNN{}s.

\paragraph{Layers.}
We also investigated the effect of the increasing number of layers $\ell$ on MRR.
We performed experiments on
\texttt{tgbl-wiki} with the number of layers increasing from 1 to 8.
As presented in \Cref{fig:layer-results},
the highest MRR for the global model is obtained when $\ell=5$ and for the local model when $\ell=7$.
Interestingly, for $\ell=5$ (which is optimal for the global model), MRR for both models is almost the same.
This, again, aligns with our theoretical results, showing that local \MPTGNN{}s are more expressive than global.

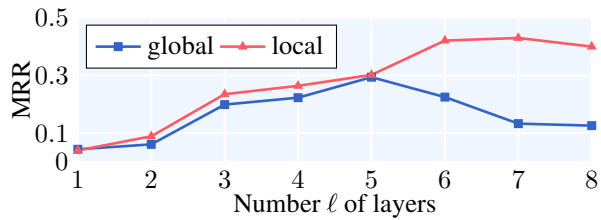
\begin{figure}[ht]

\begin{tikzpicture}
\begin{axis}[
xshift=7cm,
yshift=0cm,
width =\linewidth,
height = 3.5cm,
 xmin=1, xmax=8,
 xtick={1, 2, 3, 4, 5, 6, 7, 8},
 ymin=0, ymax=0.5,
 ytick={0.0, 0.1, 0.3, 0.5},
axis line style={white},
tick style={white},
axis background/.style={fill=Blue4},
grid = both,
major grid style={white,thick},
minor grid style={draw=none},
ylabel={MRR},
ylabel style={yshift=-0.5em},
xlabel={Number $\ell$ of layers},
xlabel style={yshift=0.5em},
legend style={at={(0.26,0.8)}, anchor=center,align=center},
legend columns=3,
]

\addplot[Blue5, mark=square*, line width=1pt, mark size=1.2pt] table {layersT1.txt};
\addplot[Red1,mark=triangle*, line width=1pt, mark size=1.2pt] table {layersT2.txt};

\legend{global, local};
\end{axis}

\end{tikzpicture}	
\caption{MRR against number of layers on \texttt{tgbl-wiki}}
\label{fig:layer-results}
\end{figure}

\paragraph{Variations.}
The lack of node features in benchmarks led us to try both random node features, which did not significantly alter results, and (transductively) learnable node features, which caused drastic overfitting.

\section{Conclusions}
We have categorised temporal message-passing graph neural networks into global and local, depending on their temporal message passing mechanism.
One might expect that  global models have higher expressive power than local, but  surprisingly we find that the two are incomparable. Further, if node colours (feature vectors) do not change over  time,  local models are strictly more powerful than global.
Our experimental results align with the theoretical findings, showing that local models obtain higher performance on temporal link-prediction tasks.

\bibliography{TGNNbiblio}

\begin{thebibliography}{35}
\providecommand{\natexlab}[1]{#1}

\bibitem[{Babai and Kucera(1979)}]{DBLP:conf/focs/BabaiK79}
Babai, L.; and Kucera, L. 1979.
\newblock Canonical Labelling of Graphs in Linear Average Time.
\newblock In \emph{20th Annual Symposium on Foundations of Computer Science,
  San Juan, Puerto Rico, 29-31 October 1979}, 39--46. {IEEE} Computer Society.

\bibitem[{Barcel{\'{o}} et~al.(2022)Barcel{\'{o}}, Galkin, Morris, and
  Orth}]{DBLP:conf/log/Barcelo00O22}
Barcel{\'{o}}, P.; Galkin, M.; Morris, C.; and Orth, M. A.~R. 2022.
\newblock Weisfeiler and Leman Go Relational.
\newblock In Rieck, B.; and Pascanu, R., eds., \emph{Learning on Graphs
  Conference, LoG 2022, 9-12 December 2022, Virtual Event}, volume 198 of
  \emph{Proceedings of Machine Learning Research}, 46. {PMLR}.

\bibitem[{Barcel{\'{o}} et~al.(2020)Barcel{\'{o}}, Kostylev, Monet,
  P{\'{e}}rez, Reutter, and Silva}]{DBLP:conf/iclr/BarceloKM0RS20}
Barcel{\'{o}}, P.; Kostylev, E.~V.; Monet, M.; P{\'{e}}rez, J.; Reutter, J.~L.;
  and Silva, J.~P. 2020.
\newblock The Logical Expressiveness of Graph Neural Networks.
\newblock In \emph{8th International Conference on Learning Representations,
  {ICLR} 2020, Addis Ababa, Ethiopia, April 26-30, 2020}. OpenReview.net.

\bibitem[{Beddar{-}Wiesing et~al.(2024)Beddar{-}Wiesing, D'Inverno, Graziani,
  Lachi, Moallemy{-}Oureh, Scarselli, and
  Thomas}]{DBLP:journals/nn/BeddarWiesingDGLMST24}
Beddar{-}Wiesing, S.; D'Inverno, G.~A.; Graziani, C.; Lachi, V.;
  Moallemy{-}Oureh, A.; Scarselli, F.; and Thomas, J.~M. 2024.
\newblock Weisfeiler-Lehman goes dynamic: An analysis of the expressive power
  of Graph Neural Networks for attributed and dynamic graphs.
\newblock \emph{Neural Networks}, 173: 106213.

\bibitem[{Cai, F{\"{u}}rer, and
  Immerman(1992)}]{DBLP:journals/combinatorica/CaiFI92}
Cai, J.; F{\"{u}}rer, M.; and Immerman, N. 1992.
\newblock An optimal lower bound on the number of variables for graph
  identification.
\newblock \emph{Comb.}, 12(4): 389--410.

\bibitem[{Chen and Wang(2023)}]{DBLP:conf/nips/ChenW23}
Chen, Y.; and Wang, D. 2023.
\newblock Calibrate and Boost Logical Expressiveness of {GNN} Over
  Multi-Relational and Temporal Graphs.
\newblock In Oh, A.; Naumann, T.; Globerson, A.; Saenko, K.; Hardt, M.; and
  Levine, S., eds., \emph{Advances in Neural Information Processing Systems 36:
  Annual Conference on Neural Information Processing Systems 2023, NeurIPS
  2023, New Orleans, LA, USA, December 10 - 16, 2023}.

\bibitem[{Gao and Ribeiro(2022)}]{DBLP:conf/icml/Gao022}
Gao, J.; and Ribeiro, B. 2022.
\newblock On the Equivalence Between Temporal and Static Equivariant Graph
  Representations.
\newblock In Chaudhuri, K.; Jegelka, S.; Song, L.; Szepesv{\'{a}}ri, C.; Niu,
  G.; and Sabato, S., eds., \emph{International Conference on Machine Learning,
  {ICML} 2022, 17-23 July 2022, Baltimore, Maryland, {USA}}, volume 162 of
  \emph{Proceedings of Machine Learning Research}, 7052--7076. {PMLR}.

\bibitem[{Gastinger et~al.(2024)Gastinger, Huang, Galkin, Loghmani, Parviz,
  Poursafaei, Danovitch, Rossi, Koutis, Stuckenschmidt
  et~al.}]{gastinger2024tgb}
Gastinger, J.; Huang, S.; Galkin, M.; Loghmani, E.; Parviz, A.; Poursafaei, F.;
  Danovitch, J.; Rossi, E.; Koutis, I.; Stuckenschmidt, H.; et~al. 2024.
\newblock TGB 2.0: A Benchmark for Learning on Temporal Knowledge Graphs and
  Heterogeneous Graphs.
\newblock \emph{arXiv preprint arXiv:2406.09639}.

\bibitem[{Gilmer et~al.(2017)Gilmer, Schoenholz, Riley, Vinyals, and
  Dahl}]{DBLP:conf/icml/GilmerSRVD17}
Gilmer, J.; Schoenholz, S.~S.; Riley, P.~F.; Vinyals, O.; and Dahl, G.~E. 2017.
\newblock Neural Message Passing for Quantum Chemistry.
\newblock In Precup, D.; and Teh, Y.~W., eds., \emph{Proceedings of the 34th
  International Conference on Machine Learning, {ICML} 2017, Sydney, NSW,
  Australia, 6-11 August 2017}, volume~70 of \emph{Proceedings of Machine
  Learning Research}, 1263--1272. {PMLR}.

\bibitem[{Grohe(2023)}]{DBLP:conf/lics/Grohe23}
Grohe, M. 2023.
\newblock The Descriptive Complexity of Graph Neural Networks.
\newblock In \emph{38th Annual {ACM/IEEE} Symposium on Logic in Computer
  Science, {LICS} 2023, Boston, MA, USA, June 26-29, 2023}, 1--14. {IEEE}.

\bibitem[{Huang et~al.(2023)Huang, Romero, Ceylan, and
  Barcel{\'{o}}}]{DBLP:conf/nips/Huang0CB23}
Huang, X.; Romero, M.; Ceylan, {\.I}.~{\.I}.; and Barcel{\'{o}}, P. 2023.
\newblock A Theory of Link Prediction via Relational Weisfeiler-Leman on
  Knowledge Graphs.
\newblock In Oh, A.; Naumann, T.; Globerson, A.; Saenko, K.; Hardt, M.; and
  Levine, S., eds., \emph{Advances in Neural Information Processing Systems 36:
  Annual Conference on Neural Information Processing Systems 2023, NeurIPS
  2023, New Orleans, LA, USA, December 10 - 16, 2023}.

\bibitem[{Ioffe and Szegedy(2015)}]{DBLP:conf/icml/IoffeS15}
Ioffe, S.; and Szegedy, C. 2015.
\newblock Batch Normalization: Accelerating Deep Network Training by Reducing
  Internal Covariate Shift.
\newblock In Bach, F.~R.; and Blei, D.~M., eds., \emph{Proceedings of the 32nd
  International Conference on Machine Learning, {ICML} 2015, Lille, France,
  6-11 July 2015}, volume~37 of \emph{{JMLR} Workshop and Conference
  Proceedings}, 448--456. JMLR.org.

\bibitem[{Kapoor et~al.(2020)Kapoor, Ben, Liu, Perozzi, Barnes, Blais, and
  O'Banion}]{DBLP:journals/corr/abs-2007-03113}
Kapoor, A.; Ben, X.; Liu, L.; Perozzi, B.; Barnes, M.; Blais, M.; and O'Banion,
  S. 2020.
\newblock Examining {COVID-19} Forecasting using Spatio-Temporal Graph Neural
  Networks.
\newblock \emph{CoRR}, abs/2007.03113.

\bibitem[{Kazemi et~al.(2019)Kazemi, Goel, Eghbali, Ramanan, Sahota, Thakur,
  Wu, Smyth, Poupart, and Brubaker}]{DBLP:journals/corr/abs-1907-05321}
Kazemi, S.~M.; Goel, R.; Eghbali, S.; Ramanan, J.; Sahota, J.; Thakur, S.; Wu,
  S.; Smyth, C.; Poupart, P.; and Brubaker, M.~A. 2019.
\newblock Time2Vec: Learning a Vector Representation of Time.
\newblock \emph{CoRR}, abs/1907.05321.

\bibitem[{Kazemi et~al.(2020)Kazemi, Goel, Jain, Kobyzev, Sethi, Forsyth, and
  Poupart}]{DBLP:journals/jmlr/KazemiGJKSFP20}
Kazemi, S.~M.; Goel, R.; Jain, K.; Kobyzev, I.; Sethi, A.; Forsyth, P.; and
  Poupart, P. 2020.
\newblock Representation Learning for Dynamic Graphs: {A} Survey.
\newblock \emph{J. Mach. Learn. Res.}, 21: 70:1--70:73.

\bibitem[{Kingma and Ba(2015)}]{DBLP:journals/corr/KingmaB14}
Kingma, D.~P.; and Ba, J. 2015.
\newblock Adam: {A} Method for Stochastic Optimization.
\newblock In Bengio, Y.; and LeCun, Y., eds., \emph{3rd International
  Conference on Learning Representations, {ICLR} 2015, San Diego, CA, USA, May
  7-9, 2015, Conference Track Proceedings}.

\bibitem[{Longa et~al.(2023)Longa, Lachi, Santin, Bianchini, Lepri, Lio, franco
  scarselli, and Passerini}]{longa2023graph}
Longa, A.; Lachi, V.; Santin, G.; Bianchini, M.; Lepri, B.; Lio, P.; franco
  scarselli; and Passerini, A. 2023.
\newblock Graph Neural Networks for Temporal Graphs: State of the Art, Open
  Challenges, and Opportunities.
\newblock \emph{Transactions on Machine Learning Research}.

\bibitem[{Luo and Li(2022)}]{DBLP:conf/log/LuoL22}
Luo, Y.; and Li, P. 2022.
\newblock Neighborhood-Aware Scalable Temporal Network Representation Learning.
\newblock In Rieck, B.; and Pascanu, R., eds., \emph{Learning on Graphs
  Conference, LoG 2022, 9-12 December 2022, Virtual Event}, volume 198 of
  \emph{Proceedings of Machine Learning Research}, 1. {PMLR}.

\bibitem[{Morris et~al.(2019)Morris, Ritzert, Fey, Hamilton, Lenssen, Rattan,
  and Grohe}]{DBLP:conf/aaai/0001RFHLRG19}
Morris, C.; Ritzert, M.; Fey, M.; Hamilton, W.~L.; Lenssen, J.~E.; Rattan, G.;
  and Grohe, M. 2019.
\newblock Weisfeiler and Leman Go Neural: Higher-Order Graph Neural Networks.
\newblock In \emph{The Thirty-Third {AAAI} Conference on Artificial
  Intelligence, {AAAI} 2019, The Thirty-First Innovative Applications of
  Artificial Intelligence Conference, {IAAI} 2019, The Ninth {AAAI} Symposium
  on Educational Advances in Artificial Intelligence, {EAAI} 2019, Honolulu,
  Hawaii, USA, January 27 - February 1, 2019}, 4602--4609. {AAAI} Press.

\bibitem[{Pareja et~al.(2020)Pareja, Domeniconi, Chen, Ma, Suzumura, Kanezashi,
  Kaler, Schardl, and Leiserson}]{DBLP:conf/aaai/ParejaDCMSKKSL20}
Pareja, A.; Domeniconi, G.; Chen, J.; Ma, T.; Suzumura, T.; Kanezashi, H.;
  Kaler, T.; Schardl, T.~B.; and Leiserson, C.~E. 2020.
\newblock EvolveGCN: Evolving Graph Convolutional Networks for Dynamic Graphs.
\newblock In \emph{The Thirty-Fourth {AAAI} Conference on Artificial
  Intelligence, {AAAI} 2020, The Thirty-Second Innovative Applications of
  Artificial Intelligence Conference, {IAAI} 2020, The Tenth {AAAI} Symposium
  on Educational Advances in Artificial Intelligence, {EAAI} 2020, New York,
  NY, USA, February 7-12, 2020}, 5363--5370. {AAAI} Press.

\bibitem[{Paszke et~al.(2019)Paszke, Gross, Massa, Lerer, Bradbury, Chanan,
  Killeen, Lin, Gimelshein, Antiga, Desmaison, K{\"{o}}pf, Yang, DeVito,
  Raison, Tejani, Chilamkurthy, Steiner, Fang, Bai, and
  Chintala}]{DBLP:conf/nips/PaszkeGMLBCKLGA19}
Paszke, A.; Gross, S.; Massa, F.; Lerer, A.; Bradbury, J.; Chanan, G.; Killeen,
  T.; Lin, Z.; Gimelshein, N.; Antiga, L.; Desmaison, A.; K{\"{o}}pf, A.; Yang,
  E.~Z.; DeVito, Z.; Raison, M.; Tejani, A.; Chilamkurthy, S.; Steiner, B.;
  Fang, L.; Bai, J.; and Chintala, S. 2019.
\newblock PyTorch: An Imperative Style, High-Performance Deep Learning Library.
\newblock In Wallach, H.~M.; Larochelle, H.; Beygelzimer, A.;
  d'Alch{\'{e}}{-}Buc, F.; Fox, E.~B.; and Garnett, R., eds., \emph{Advances in
  Neural Information Processing Systems 32: Annual Conference on Neural
  Information Processing Systems 2019, NeurIPS 2019, December 8-14, 2019,
  Vancouver, BC, Canada}, 8024--8035.

\bibitem[{Qin and Yeung(2024)}]{DBLP:journals/csur/QinY24}
Qin, M.; and Yeung, D. 2024.
\newblock Temporal Link Prediction: {A} Unified Framework, Taxonomy, and
  Review.
\newblock \emph{{ACM} Comput. Surv.}, 56(4): 89:1--89:40.

\bibitem[{Qu et~al.(2020)Qu, Zhu, Duan, and Shi}]{DBLP:conf/www/QuZDS20}
Qu, L.; Zhu, H.; Duan, Q.; and Shi, Y. 2020.
\newblock Continuous-Time Link Prediction via Temporal Dependent Graph Neural
  Network.
\newblock In Huang, Y.; King, I.; Liu, T.; and van Steen, M., eds., \emph{{WWW}
  '20: The Web Conference 2020, Taipei, Taiwan, April 20-24, 2020}, 3026--3032.
  {ACM} / {IW3C2}.

\bibitem[{Rossi et~al.(2020)Rossi, Chamberlain, Frasca, Eynard, Monti, and
  Bronstein}]{DBLP:journals/corr/abs-2006-10637}
Rossi, E.; Chamberlain, B.; Frasca, F.; Eynard, D.; Monti, F.; and Bronstein,
  M.~M. 2020.
\newblock Temporal Graph Networks for Deep Learning on Dynamic Graphs.
\newblock \emph{CoRR}, abs/2006.10637.

\bibitem[{Rossi and Ahmed(2015)}]{DBLP:conf/aaai/RossiA15}
Rossi, R.~A.; and Ahmed, N.~K. 2015.
\newblock The Network Data Repository with Interactive Graph Analytics and
  Visualization.
\newblock In Bonet, B.; and Koenig, S., eds., \emph{Proceedings of the
  Twenty-Ninth {AAAI} Conference on Artificial Intelligence, January 25-30,
  2015, Austin, Texas, {USA}}, 4292--4293. {AAAI} Press.

\bibitem[{Schlichtkrull et~al.(2018)Schlichtkrull, Kipf, Bloem, van~den Berg,
  Titov, and Welling}]{DBLP:conf/esws/SchlichtkrullKB18}
Schlichtkrull, M.~S.; Kipf, T.~N.; Bloem, P.; van~den Berg, R.; Titov, I.; and
  Welling, M. 2018.
\newblock Modeling Relational Data with Graph Convolutional Networks.
\newblock In Gangemi, A.; Navigli, R.; Vidal, M.; Hitzler, P.; Troncy, R.;
  Hollink, L.; Tordai, A.; and Alam, M., eds., \emph{The Semantic Web - 15th
  International Conference, {ESWC} 2018, Heraklion, Crete, Greece, June 3-7,
  2018, Proceedings}, volume 10843 of \emph{Lecture Notes in Computer Science},
  593--607. Springer.

\bibitem[{Skarding, Gabrys, and
  Musial(2021)}]{DBLP:journals/access/SkardingGM21}
Skarding, J.; Gabrys, B.; and Musial, K. 2021.
\newblock Foundations and Modeling of Dynamic Networks Using Dynamic Graph
  Neural Networks: {A} Survey.
\newblock \emph{{IEEE} Access}, 9: 79143--79168.

\bibitem[{Souza et~al.(2022)Souza, Mesquita, Kaski, and
  Garg}]{DBLP:conf/nips/SouzaMKG22}
Souza, A.~H.; Mesquita, D.; Kaski, S.; and Garg, V.~K. 2022.
\newblock Provably expressive temporal graph networks.
\newblock In Koyejo, S.; Mohamed, S.; Agarwal, A.; Belgrave, D.; Cho, K.; and
  Oh, A., eds., \emph{Advances in Neural Information Processing Systems 35:
  Annual Conference on Neural Information Processing Systems 2022, NeurIPS
  2022, New Orleans, LA, USA, November 28 - December 9, 2022}.

\bibitem[{Vashishth et~al.(2020)Vashishth, Sanyal, Nitin, and
  Talukdar}]{DBLP:conf/iclr/VashishthSNT20}
Vashishth, S.; Sanyal, S.; Nitin, V.; and Talukdar, P.~P. 2020.
\newblock Composition-based Multi-Relational Graph Convolutional Networks.
\newblock In \emph{8th International Conference on Learning Representations,
  {ICLR} 2020, Addis Ababa, Ethiopia, April 26-30, 2020}. OpenReview.net.

\bibitem[{Weisfeiler and Leman(1968)}]{weisfeiler1968reduction}
Weisfeiler, B.; and Leman, A. 1968.
\newblock The reduction of a graph to canonical form and the algebra which
  appears therein.
\newblock \emph{nti, Series}, 2(9): 12--16.

\bibitem[{Wu et~al.(2023)Wu, Sun, Zhang, Xie, and
  Cui}]{DBLP:journals/csur/WuSZXC23}
Wu, S.; Sun, F.; Zhang, W.; Xie, X.; and Cui, B. 2023.
\newblock Graph Neural Networks in Recommender Systems: {A} Survey.
\newblock \emph{{ACM} Comput. Surv.}, 55(5): 97:1--97:37.

\bibitem[{Xu et~al.(2020)Xu, Ruan, K{\"{o}}rpeoglu, Kumar, and
  Achan}]{DBLP:conf/iclr/XuRKKA20}
Xu, D.; Ruan, C.; K{\"{o}}rpeoglu, E.; Kumar, S.; and Achan, K. 2020.
\newblock Inductive representation learning on temporal graphs.
\newblock In \emph{8th International Conference on Learning Representations,
  {ICLR} 2020, Addis Ababa, Ethiopia, April 26-30, 2020}. OpenReview.net.

\bibitem[{Xu et~al.(2019)Xu, Hu, Leskovec, and
  Jegelka}]{DBLP:conf/iclr/XuHLJ19}
Xu, K.; Hu, W.; Leskovec, J.; and Jegelka, S. 2019.
\newblock How Powerful are Graph Neural Networks?
\newblock In \emph{7th International Conference on Learning Representations,
  {ICLR} 2019, New Orleans, LA, USA, May 6-9, 2019}. OpenReview.net.

\bibitem[{Yu, Yin, and Zhu(2018)}]{DBLP:conf/ijcai/YuYZ18}
Yu, B.; Yin, H.; and Zhu, Z. 2018.
\newblock Spatio-Temporal Graph Convolutional Networks: {A} Deep Learning
  Framework for Traffic Forecasting.
\newblock In Lang, J., ed., \emph{Proceedings of the Twenty-Seventh
  International Joint Conference on Artificial Intelligence, {IJCAI} 2018, July
  13-19, 2018, Stockholm, Sweden}, 3634--3640. ijcai.org.

\bibitem[{Zhou et~al.(2020)Zhou, Cui, Hu, Zhang, Yang, Liu, Wang, Li, and
  Sun}]{DBLP:journals/aiopen/ZhouCHZYLWLS20}
Zhou, J.; Cui, G.; Hu, S.; Zhang, Z.; Yang, C.; Liu, Z.; Wang, L.; Li, C.; and
  Sun, M. 2020.
\newblock Graph neural networks: {A} review of methods and applications.
\newblock \emph{{AI} Open}, 1: 57--81.

\end{thebibliography}

\onecolumn

\section{Appendix}

\section{Types of Message-Passing}

We will introduce several \MPTGNN{} models and provide their formalisations in a common format. 
This will allow us to observe similarities and differences between these approaches, and study their expressive power in later sections.

\paragraph{MP-TGN.}
A temporal message-passing component of TGN \cite{DBLP:journals/corr/abs-2006-10637}, also known as Temporal Graph Sum. 
Although introduced to compute embeddings $\h_v^{(\ell)}(t)$ for event-based representation of temporal graphs (with edges labelled by feature vectors) 
it also can be used for snapshot temporal graphs, as described by the equations:
\begin{align*}
\h_v^{(0)}(t) &= \x_v(t),
\\
\tilde{\h}_v^{(\ell)}(t) & = \relu \Big( \sum_{(u,t') \in \NH(v,t) } 
 \W_1^{(\ell)}
( \mathbf{h}_u^{(\ell-1)}(t) \con g(t-t')) \Big), 
\\
\h_v^{(\ell)}(t) & = \W_2^{(\ell)} \Big( \h_v^{(\ell-1)}(t)\con \tilde{\h}_v^{(\ell)}(t) \Big),
\end{align*}
where  $\ell \in \{ 1, \dots L \}$, 
\con{} is concatenation,
$g$ is a function mapping time distances into vectors (e.g., based on Time2Vec \cite{DBLP:journals/corr/abs-1907-05321}),
$\relu$ is the rectified linear unit,  
and $\W_1^{(\ell)}$ and $\W_2^{(\ell)}$ are learnable matrices.
We can generalise this form, by allowing for arbitrary combination $\com$ and aggregation $\agg$ functions, which leads to the following form \cite{DBLP:conf/nips/SouzaMKG22}:
\begin{align*}
\h_v^{(0)}(t) &= \x_v(t), 
\\
\h_v^{(\ell)}(t) &= \com^{(\ell)} \Big(
\h_v^{(\ell-1)}(t),
\agg^{(\ell)}(
 \ldblbrace
( \mathbf{h}_u^{(\ell-1)}(t), g(t-t') ) \mid
 (u,t') \in \NH(v,t) 
\rdblbrace
)
\Big), 
\end{align*}
where $\ldblbrace \cdot \rdblbrace$ stands for a multiset (generalisation of a set, where the same element can occur multiple time).

\paragraph{TE.}
Another temporal message-passing mechanism is used in Temporal Embedding (TE)  model \cite{longa2023graph}, which generalise such models as TGAT \cite{DBLP:conf/iclr/XuRKKA20} and NAT \cite{DBLP:conf/log/LuoL22}.
Embeddings are computed in TE  as follows:
\begin{align*}
\h_v^{(0)}(t) &= \x_v(t), 
\\
\h_v^{(\ell)}(t) &= \com^{(\ell)} \Big(
\h_v^{(\ell-1)}(t),
\agg^{(\ell)}(
\ldblbrace
( \h_u^{(\ell-1)}(t'), g(t-t') ) \mid
 (u,t') \in \NH(v,t) 
\rdblbrace
)
\Big).
\end{align*}
The difference in computing  $\h_v^{(\ell)}(t)$ by MP-TGN and TE is that MP-TGN
aggregates embeddings at $t$ (expressions $\h_u^{(\ell-1)}(t)$),
whereas TE aggregates embeddings at $t'$ (expressions $\h_u^{(\ell-1)}(t')$).
This subtle difference, although can appear to be insignificant, leads to very different behaviour of MP-TGN and TE, as we will show in the later parts of the paper.

\pwline{ Can you Michael check the papers about TGAT and NAT if the above formulation of embeddings' computation indeed generalises the approaches used in TGAT and NAT?}
\mrline{Looks OK to me. There is a distinction between ``$t'$ is the time of a connection event between $u$ and $v$''~\cite{DBLP:journals/datamine/LongaCLP22} and $(u, t') \in \NH_{\TG}(v, t)$ --- but I don't think it makes a huge difference as you can rectify it easily in either direction.}

\paragraph{TDGNN.}
Yet another approach for temporal message-passing is proposed in TDGNN model  \cite{DBLP:conf/www/QuZDS20}.
It takes as an input a temporal graph in the aggregated form, so each vertex $v$ has a single initial feature vector $\x_v$.
This is in contrast to MP-TGN and TE, where  $v$ can be assigned a different feature vector $\x_v(t)$ in every time point $t$.
Then, TDGNN computes embeddings as follows:
\begin{align*}
\h_v^{(0)}(t) &= \x_v,
\\
\h_v^{(\ell)}(t) & = \relu \Big( \sum_{(u,t') \in \NH^+(v,t) }  \alpha_{u,v}^{t'} \W^{(\ell)} \h_u^{(\ell-1)}(t) \Big),
\\
\alpha_{u,v}^{t'} & =
\frac{e^{t-t'}}{\sum_{(u,t') \in \NH^+(v,t) } e^{t-t'}}, 
\end{align*}
where $\W^{(\ell)}$ is a learnable matrix and $\NH^+(v,t)$ is the ``reflexive'' extension of $\NH(v,t)$, namely $\NH^+(v,t) = \NH(v,t) \cup \{ (v,t) \}$.
Hence, the general form of TDGNN  can be written as follows:
\begin{align*}
\h_v^{(0)}(t) &= \x_v,
\\
\h_v^{(\ell)}(t) &= \com^{(\ell)} \Big(
\h_v^{(\ell-1)}(t),
\agg^{(\ell)}(
\ldblbrace
( \h_u^{(\ell-1)}(t'), g(t-t') ) \mid
 (u,t') \in \NH(v,t) 
\rdblbrace
)
\Big).
\end{align*}
Hence, TDGNN message-passing approach is similar to the one used in MP-TGN, except that it uses the same feature vector $\x_v$ for $v$ in all time points.

\pwline{Code for TDGNN is here \url{https://github.com/Leo-Q-316/TDGNN}.}

\section{Comparison of global and local \MPTGNN{}s}

It is important to observe that the types of message-passing mechanisms in  global and local \MPTGNN{}s are very different.
For example  to compute $\h_b(t_4)$ in $\TG$ from \Cref{fig:snapshot}, a global \MPTGNN{}s   aggregates $\h_c(t_4)$, $\h_c(t_3)$, and $\h_a(t_2)$, whereas a local \MPTGNN{}s   aggregates  $\h_c(t_4)$, $\h_c(t_4)$, and $\h_a(t_4)$.
Surprisingly, relative expressiveness and performance of such \MPTGNN{}s have not been thoroughly studied so far. 
We aim to fill this gap by, among others, exploiting Weisfeiler-Leman-like algorithms.

\section{Proofs details for  Temporal Weisfeiler-Leman Characterisation}
\Konei*
\begin{proof}
The proof is by induction on the number $\ell \in \N$.
For the basis, it suffices to show that 
$\rwl^{(0)}(v,t_i)=\h_v^{(0)}(t_i)$, for each timestamped node $(v,t_i)$ in $\TG$.
This holds since  $\rwl^{(0)}(v,t_i)=c(v,t_i)=c_i(v)=\h_v^{(0)}(t_i)$; these equalities hold, respectively, by the definition of $\rwl$, by the definition of $\Kone{\TG}$, and by Equation~\eqref{TGN1}.

For the inductive step  assume that the implication holds for $\ell-1$;  we will show it for $\ell$.
To show the implication  assume that 
$\rwl^{(\ell)}(v,t) = \rwl^{(\ell)}(u,t')$, for some $(v,t)$ and $(u,t')$.
Thus, by the definition of $\rwl$,
we have ($\ast$) that 
$\rwl^{(\ell-1)}(v,t) = \rwl^{(\ell-1)}(u,t')$
and  ($\ast\ast$) that the multiset 
$
\ldblbrace \rwl^{(\ell-1)}(w,t'') \mid (w,t'') \in \NH_r(v,t), r \in R \rdblbrace
$
equals the multiset
$ 
\ldblbrace \rwl^{(\ell-1)}(w,t'') \mid (w,t'') \in \NH_r(u,t'), r \in R \rdblbrace 
$.
By the inductive assumption, Statement~($\ast$) implies that $\h_v^{(\ell-1)}(t)=\h_u^{(\ell-1)}(t')$.
Next, we observe a crucial property of $\Kone{\TG}$; its definition implies that, for any $(w,t'')$ and $(v,t)$, the fact that $(w,t'') \in \NH_r(v,t)$ in $\Kone{\TG}$ is equivalent to  $(w,t'') \in \NH(v,t)$ and $r=t-t''$ in $\TG$.
Hence, Statement~($\ast\ast$) and the inductive assumption   imply that in $\TG$, the multiset 
$ 
\ldblbrace
( \h_w^{(\ell-1)}(t''), t-t' )\mid
 (w,t'') \in \NH(v,t) 
\rdblbrace
$
equals
$
\ldblbrace
( \h_w^{(\ell-1)}(t''), t'-t'' ) \mid
 (w,t'') \in \NH(u,t') 
\rdblbrace
$.
So, by Equation~\eqref{TGN2} (with $\fone=\h_u^{(\ell-1)}(t')$), we get
$\h_v^{(\ell)}(t)=\h_u^{(\ell)}(t')$.
\end{proof}

\Koneii*
\begin{proof}
Note that by \Cref{thm:Kone1}, the second statement implies the first one.
To show the opposite implication we will use the result of
\citet{DBLP:conf/nips/Huang0CB23}[Theorem A.1], who showed that for any knowledge graph $\KG=(V,E,R,c)$ and any $L \in \N$  there is a relational message-passing neural network R-MPNN model $\modelB$ (R-MPNNs are  defined by \citet{DBLP:conf/nips/Huang0CB23} in Section 3.2) with $L$ layers, such that for any nodes $v,u \in V$ and any $\ell \leq L$, if
$\hh_v^{(\ell)} = \hh_u^{(\ell)}$ in $\modelB$, then
$\rwl^{(\ell)}(v) = \rwl^{(\ell)}(u)$ in $\KG$.
In particular, they showed existence of an R-MPNN model of the following form (for any $v \in V$ and $\ell \leq L$):
\begin{align*}
\hh_v^{(0)} & = c(v),
\\
\hh_v^{(\ell)} & = \sign \big( 
\W^{(\ell)} + (\hh_v^{(\ell-1)} + \sum_{r\in R} \; \sum_{u \in \NH_{r}(v)} 
\alpha_r \hh_u^{(\ell-1)} ) - \bb \big),
\end{align*}
where $\W^{(\ell)}$ is a parameter matrix, $\alpha_r$ is a parameter, and $\bb$ is a bias term (they used the all-ones vector $\bb=\mathbf{1}$). 

We will use the result of \citet{DBLP:conf/nips/Huang0CB23}  to prove \Cref{thm:Kone2}.
To this end, let us fix a  labelled temporal graph $\TG = (\G_1,t_1), \dots, (\G_n,t_n)$, with  $\G_i=(V_i,E_i,c_i)$, and  $L \in \N$.
Now, we consider the knowledge graph $\Kone{\TG} = (V,E,R,c)$, where $V$ is the set of timestamped nodes in $\TG$.
By the result of \citet{DBLP:conf/nips/Huang0CB23}, there exists an R-MPNN model $\modelB$ of the form 
\begin{align*}
\hh_{(v,t)}^{(0)} & = c(v,t),
\\
\hh_{(v,t)}^{(\ell)} & = \sign \Big( 
\W^{(\ell)} (\hh_{(v,t)}^{(\ell-1)} + \sum_{r\in R} \; \; \sum_{(u,t') \in \NH_{r}(v,t)} 
\alpha_r \hh_{(u,t')}^{(\ell-1)} ) - \bb \Big),
\end{align*}
such that for 
any nodes $(v,t),(u,t') \in V$ and any $\ell \leq L$, if
$\hh_{(v,t)}^{(\ell)} = \hh_{(u,t')}^{(\ell)}$ in $\modelB$, then
$\rwl^{(\ell)}(v,t) = \rwl^{(\ell)}(u,t')$ in $\Kone{\TG}$.

We will use $\modelB$ to construct required $\model$.
In particular, it suffices to construct a global \MPTGNN{} $\model$ with $L$ layers which on $\TG$ computes embeddings $\h^{(\ell)}_v(t)$ that coincide with the corresponding embeddings $\hh_{(v,t)}^{(\ell)}$ computed by $\modelB$ on $\Kone{\TG}$.
We observe that, by the definition of $\Kone{\TG}$, we have $u \in \NH_{r}(v)$ in $\Kone{\TG}$ if and only if 
$(u,t') \in \NH(v,t)$ and $r=t-t'$ in $\TG$.
Hence, we can obtain required $\model$ as follows:
\begin{align*}
\h_v^{(0)}(t) & = c(v,t),
\\
\h_v^{(\ell)}(t) & = \sign \Big( 
\W^{(\ell)} (\h_v^{(\ell-1)}(t) + \aggsum
\ldblbrace
( 
\alpha_{(t-t') }\mathbf{h}_u^{(\ell-1)}(t') ) \mid
 (u,t') \in \NH(v,t) 
\rdblbrace ) 
- \bb \Big),
\end{align*}
where $\W^{(\ell)} $, $\alpha_r$, and $\bb$ are as in $\modelB$.
It remains to show that $\model$ can be written as a global \MPTGNN{}.
This is indeed the case, since $\model$ can be written as
\begin{align*}
\h_v^{(0)}(t) & = c(v,t),
\\
\h_v^{(\ell)}(t) & = \com^{(\ell)} \Big(
\h_v^{(\ell-1)}(t),
\agg^{(\ell)}(
 \ldblbrace
( \h_u^{(\ell)}(t'), g(t-t') ) \mid
 (u,t') \in \NH(v,t) 
\rdblbrace
)
\Big),  
\end{align*}
where:
\begin{align*}
\com^{(\ell)} (\h,\h') &= \sign(\W^{(\ell)}(\h + \h') - \bb),
\\
\agg^{(\ell)}(
 \ldblbrace
( \mathbf{h}_u^{(\ell-1)}(t'), g(t-t') ) \mid
 (u,t') \in \NH(v,t) 
\rdblbrace
) & = 
\aggsum
\ldblbrace
( 
\alpha_{(t-t') }\mathbf{h}_u^{(\ell-1)}(t') ) \mid
 (u,t') \in \NH(v,t) 
\rdblbrace.
\end{align*}
\end{proof}

\Ktwoi*
\begin{proof}
The proof structure is similar to the one for \Cref{thm:Kone1}.
The main difference is in the inductive step, which  exploits now the following property  of  $\Ktwo{\TG}$:
we have 
$(u,t') \in \NH_{r}(v,t)$ in $\Ktwo{\TG}$ if and only if  $t'=t$ and there exists a timestamped node
$(u,t'')$ in $\TG$
such that 
$(u,t'') \in \NH(v,t)$ and $r=t-t''$.
This, by the form of $\rwl$ and message-passing in local \MPTGNN{}s, allows us to show that $\rwl^{(\ell)}(v,t) = \rwl^{(\ell)}(u,t')$  implies  $\h_v^{(\ell)}(t)=\h_u^{(\ell)}(t')$. 
\end{proof}

\Ktwoii*
\begin{proof}
The proof uses the R-MPNN $\modelB$ mentioned in the proof of \Cref{thm:Kone2}.
It suffices to construct a local 
\MPTGNN{}$\model$ with $L$ layers which on $\TG$ computes embeddings $\h^{(\ell)}_v(t)$ that coincide with the corresponding embeddings $\hh_{(v,t)}^{(\ell)}$ computed by $\modelB$ on $\Ktwo{\TG}$.
This is obtained by the model $\model$ performing the following computations:
\begin{align*}
\h_v^{(0)}(t) & = c(v,t),
\\
\h_v^{(\ell)}(t) & = \sign \Big( 
\W^{(\ell)} (\h_v^{(\ell-1)}(t) + \aggsum
\ldblbrace
( 
\alpha_{(t-t') }\mathbf{h}_u^{(\ell-1)}(t) ) \mid
 (u,t') \in \NH(v,t) 
\rdblbrace ) 
- \bb \Big),
\end{align*}
where $\W^{(\ell)} $, $\alpha_r$, and $\bb$ are as in $\modelB$.
Note that the difference with respect to \Cref{thm:Kone2} is that 
$\mathbf{h}_u^{(\ell-1)}(t)$ is uses instead of $\mathbf{h}_u^{(\ell-1)}(t')$ under \aggsum{} operator.
Model $\model$ can be written as a local \MPTGNN{} as follows
\begin{align*}
\h_v^{(0)}(t) & = c(v,t),
\\
\h_v^{(\ell)}(t) & = \com^{(\ell)} \Big(
\h_v^{(\ell-1)}(t),
\agg^{(\ell)}(
 \ldblbrace
( \h_u^{(\ell)}(t), g(t-t') ) \mid
 (u,t') \in \NH(v,t) 
\rdblbrace
)
\Big),  
\end{align*}
where:
\begin{align*}
\com^{(\ell)} (\h,\h') &= \sign(\W^{(\ell)}(\h + \h') - \bb),
\\
\agg^{(\ell)}(
 \ldblbrace
( \mathbf{h}_u^{(\ell-1)}(t), g(t-t') ) \mid
 (u,t') \in \NH(v,t) 
\rdblbrace
) & = 
\aggsum
\ldblbrace
( 
\alpha_{(t-t') }\mathbf{h}_u^{(\ell-1)}(t) ) \mid
 (u,t') \in \NH(v,t) 
\rdblbrace.
\end{align*}
\end{proof}

\section{Proof details for  Temporal Isomorphisms}

\piso*
\begin{proof}
Consider $\TG$ and $\TG'$ from \Cref{fig:point}, where $(a,t_2)$ is pointwise isomorphic to $(a',t_2)$.
The knowledge graphs $\Kone{\TG}$ and $\Kone{\TG'}$ constructed for them, are depicted in \Cref{fig:point_one}.
Since $(a,t_2)$ has one incoming edge in $\Kone{\TG}$, but $(a',t_2)$ has no incoming edges in $\Kone{\TG'}$, we obtain  $\rwl^{(1)}(a,t_2) \neq \rwl^{(1)}(a',t_2)$.
Hence, by \Cref{thm:Kone2}, there exists a global \MPTGNN{}s  which computes embeddings  $\h_{a}^{(1)}(t_2) \neq \h_{a'}^{(1)}(t_2)$.

\begin{figure}[H]
\centering
\begin{tikzpicture}[
dot/.style = {draw, circle, minimum size=#1,
              inner sep=0pt, outer sep=0pt},
dot/.default = 6pt
]
\scriptsize
\pgfmathsetmacro{\tline}{-0.25}
\pgfmathsetmacro{\w}{1.3}
\pgfmathsetmacro{\h}{1.5}
\pgfmathsetmacro{\inh}{0.7}
\pgfmathsetmacro{\dist}{1.7}
\pgfmathsetmacro{\Ax}{0.5}
\pgfmathsetmacro{\Ay}{1.4}
\pgfmathsetmacro{\Bx}{0.4}
\pgfmathsetmacro{\By}{0.5}
\pgfmathsetmacro{\Cx}{1}
\pgfmathsetmacro{\Cy}{1}

\draw[->] (1.8,\tline) -- (4.8,\tline);

\foreach \x in {1,...,2}
{
\draw[fill=gray!90!black,opacity=0.2] (\x*\dist,0) -- (\x*\dist+\w,\inh) -- (\x*\dist+\w,\inh+\h) -- (\x*\dist,\h) -- cycle;

\node at (\x*\dist + 0.87 * \w, \inh+ \h - 0.3) {$G_{\x}$};
\node at (\x*\dist + 0.6 * \w, 0) {$t_{\x}={\x}$};
\draw[-] (\x*\dist + 0.6 * \w, \tline+0.05) -- (\x*\dist + 0.6 * \w, \tline-0.05);

\node[dot=11pt,draw=none] (A\x) at (\x*\dist+\Ax, 0+\Ay) {};
\node[dot=11pt,draw=none] (B\x) at (\x*\dist+\Bx, 0+\By) {};
\node[dot=11pt,draw=none] (C\x) at (\x*\dist+\Cx, 0+\Cy) {};
}


\draw[<->,blue,thick] (A1) -- (B1) node[pos=0.5,left] {0};
\draw[->,blue,thick] (A1) -- (B2) node[pos=0.7,below] {1};
\draw[->,blue,thick] (B1) -- (A2) node[pos=0.7,above] {1};

\node[dot=11pt,draw=black,fill=mygreen] at (A1) {$a$};
\node[dot=11pt,draw=black,fill=mygreen] at (B1) {$b$};
\node[dot=11pt,draw=black,fill=mygreen] at (C1) {$c$};
\node[dot=11pt,draw=black,fill=mygreen] at (A2) {$a$};
\node[dot=11pt,draw=black,fill=mygreen] at (B2) {$b$};
\node[dot=11pt,draw=black,fill=mygreen] at (C2) {$c$};
\node at (3.3,2.6) {$\TG$};
\end{tikzpicture}
\qquad
\begin{tikzpicture}[
dot/.style = {draw, circle, minimum size=#1,
              inner sep=0pt, outer sep=0pt},
dot/.default = 6pt
]
\scriptsize
\pgfmathsetmacro{\tline}{-0.25}
\pgfmathsetmacro{\w}{1.3}
\pgfmathsetmacro{\h}{1.5}
\pgfmathsetmacro{\inh}{0.7}
\pgfmathsetmacro{\dist}{1.7}
\pgfmathsetmacro{\Ax}{0.5}
\pgfmathsetmacro{\Ay}{1.4}
\pgfmathsetmacro{\Bx}{0.4}
\pgfmathsetmacro{\By}{0.5}
\pgfmathsetmacro{\Cx}{1}
\pgfmathsetmacro{\Cy}{1}

\draw[->] (1.8,\tline) -- (4.8,\tline);

\foreach \x in {1,...,2}
{
\draw[fill=gray!90!black,opacity=0.2] (\x*\dist,0) -- (\x*\dist+\w,\inh) -- (\x*\dist+\w,\inh+\h) -- (\x*\dist,\h) -- cycle;

\node at (\x*\dist + 0.87 * \w, \inh+ \h - 0.3) {$G'_{\x}$};
\node at (\x*\dist + 0.6 * \w, 0) {$t_{\x}={\x}$};
\draw[-] (\x*\dist + 0.6 * \w, \tline+0.05) -- (\x*\dist + 0.6 * \w, \tline-0.05);

\node[dot=11pt,draw=none] (A\x) at (\x*\dist+\Ax, 0+\Ay) {};
\node[dot=11pt,draw=none] (B\x) at (\x*\dist+\Bx, 0+\By) {};
\node[dot=11pt,draw=none] (C\x) at (\x*\dist+\Cx, 0+\Cy) {};
}


\draw[<->,blue,thick] (B1) -- (C1) node[pos=0.3,above] {0};
\draw[->,blue,thick] (C1) -- (B2) node[pos=0.4,above] {1};
\draw[->,blue,thick] (B1) -- (C2) node[pos=0.3,below] {1};

\node[dot=11pt,draw=black,fill=mygreen] at (A1) {$a'$};
\node[dot=11pt,draw=black,fill=mygreen] at (B1) {$b'$};
\node[dot=11pt,draw=black,fill=mygreen] at (C1) {$c'$};
\node[dot=11pt,draw=black,fill=mygreen] at (A2) {$a'$};
\node[dot=11pt,draw=black,fill=mygreen] at (B2) {$b'$};
\node[dot=11pt,draw=black,fill=mygreen] at (C2) {$c'$};
\node at (3.3,2.6) {$\TG'$};
\end{tikzpicture}
\caption{$\Kone{\TG}$ and $\Kone{\TG'}$ for $\TG$  and $\TG'$  from \Cref{fig:point}
}
\label{fig:point_one}
\end{figure}

Now consider knowledge graphs $\Ktwo{\TG}$ and $\Ktwo{\TG'}$ constructed for $\TG$ and $\TG'$ from  \Cref{fig:point}; they are depicted in \Cref{fig:point_two}.
Since $(a,t_2)$ has one incoming edge in $\Kone{\TG}$, but $(a',t_2)$ has no incoming edges in $\Kone{\TG'}$, we obtain  $\rwl^{(1)}(a,t_2) \neq \rwl^{(1)}(a',t_2)$.
Hence, by \Cref{thm:Ktwo2}, there exist a local \MPTGNN{}s  which computes embeddings  $\h_{a}^{(1)}(t_2) \neq \h_{a'}^{(1)}(t_2)$.

\begin{figure}[ht]
\centering
\begin{tikzpicture}[
dot/.style = {draw, circle, minimum size=#1,
              inner sep=0pt, outer sep=0pt},
dot/.default = 6pt
]
\scriptsize
\pgfmathsetmacro{\tline}{-0.25}
\pgfmathsetmacro{\w}{1.3}
\pgfmathsetmacro{\h}{1.5}
\pgfmathsetmacro{\inh}{0.7}
\pgfmathsetmacro{\dist}{1.7}
\pgfmathsetmacro{\Ax}{0.5}
\pgfmathsetmacro{\Ay}{1.4}
\pgfmathsetmacro{\Bx}{0.4}
\pgfmathsetmacro{\By}{0.5}
\pgfmathsetmacro{\Cx}{1}
\pgfmathsetmacro{\Cy}{1}

\draw[->] (1.8,\tline) -- (4.8,\tline);

\foreach \x in {1,...,2}
{
\draw[fill=gray!90!black,opacity=0.2] (\x*\dist,0) -- (\x*\dist+\w,\inh) -- (\x*\dist+\w,\inh+\h) -- (\x*\dist,\h) -- cycle;

\node at (\x*\dist + 0.87 * \w, \inh+ \h - 0.3) {$G_{\x}$};
\node at (\x*\dist + 0.6 * \w, 0) {$t_{\x}={\x}$};
\draw[-] (\x*\dist + 0.6 * \w, \tline+0.05) -- (\x*\dist + 0.6 * \w, \tline-0.05);

\node[dot=11pt,draw=none] (A\x) at (\x*\dist+\Ax, 0+\Ay) {};
\node[dot=11pt,draw=none] (B\x) at (\x*\dist+\Bx, 0+\By) {};
\node[dot=11pt,draw=none] (C\x) at (\x*\dist+\Cx, 0+\Cy) {};
}


\draw[<->,blue,thick] (A1) -- (B1) node[pos=0.5,left] {0};
\draw[<->,blue,thick] (A2) -- (B2) node[pos=0.5,left] {1};

\node[dot=11pt,draw=black,fill=mygreen] at (A1) {$a$};
\node[dot=11pt,draw=black,fill=mygreen] at (B1) {$b$};
\node[dot=11pt,draw=black,fill=mygreen] at (C1) {$c$};
\node[dot=11pt,draw=black,fill=mygreen] at (A2) {$a$};
\node[dot=11pt,draw=black,fill=mygreen] at (B2) {$b$};
\node[dot=11pt,draw=black,fill=mygreen] at (C2) {$c$};
\node at (3.3,2.6) {$\TG$};
\end{tikzpicture}
\qquad
\begin{tikzpicture}[
dot/.style = {draw, circle, minimum size=#1,
              inner sep=0pt, outer sep=0pt},
dot/.default = 6pt
]
\scriptsize
\pgfmathsetmacro{\tline}{-0.25}
\pgfmathsetmacro{\w}{1.3}
\pgfmathsetmacro{\h}{1.5}
\pgfmathsetmacro{\inh}{0.7}
\pgfmathsetmacro{\dist}{1.7}
\pgfmathsetmacro{\Ax}{0.5}
\pgfmathsetmacro{\Ay}{1.4}
\pgfmathsetmacro{\Bx}{0.4}
\pgfmathsetmacro{\By}{0.5}
\pgfmathsetmacro{\Cx}{1}
\pgfmathsetmacro{\Cy}{1}

\draw[->] (1.8,\tline) -- (4.8,\tline);

\foreach \x in {1,...,2}
{
\draw[fill=gray!90!black,opacity=0.2] (\x*\dist,0) -- (\x*\dist+\w,\inh) -- (\x*\dist+\w,\inh+\h) -- (\x*\dist,\h) -- cycle;

\node at (\x*\dist + 0.87 * \w, \inh+ \h - 0.3) {$G'_{\x}$};
\node at (\x*\dist + 0.6 * \w, 0) {$t_{\x}={\x}$};
\draw[-] (\x*\dist + 0.6 * \w, \tline+0.05) -- (\x*\dist + 0.6 * \w, \tline-0.05);

\node[dot=11pt,draw=none] (A\x) at (\x*\dist+\Ax, 0+\Ay) {};
\node[dot=11pt,draw=none] (B\x) at (\x*\dist+\Bx, 0+\By) {};
\node[dot=11pt,draw=none] (C\x) at (\x*\dist+\Cx, 0+\Cy) {};
}


\draw[<->,blue,thick] (B1) -- (C1) node[pos=0.3,above] {0};
\draw[<->,blue,thick] (B2) -- (C2) node[pos=0.3,above] {1};

\node[dot=11pt,draw=black,fill=mygreen] at (A1) {$a'$};
\node[dot=11pt,draw=black,fill=mygreen] at (B1) {$b'$};
\node[dot=11pt,draw=black,fill=mygreen] at (C1) {$c'$};
\node[dot=11pt,draw=black,fill=mygreen] at (A2) {$a'$};
\node[dot=11pt,draw=black,fill=mygreen] at (B2) {$b'$};
\node[dot=11pt,draw=black,fill=mygreen] at (C2) {$c'$};
\node at (3.3,2.6) {$\TG'$};
\end{tikzpicture}
\caption{$\Ktwo{\TG}$ and $\Ktwo{\TG'}$ for $\TG$  and $\TG'$  from \Cref{fig:point}
}
\label{fig:point_two}
\end{figure}
\end{proof}

\iso*
\begin{proof}
Assume that  $(v,t)$ from $\TG$ and  $(u,t')$ from $\TG'$ are  timewise isomorphic.
Hence, by \Cref{def:iso}, $\TG$ and $\TG'$ are of the forms $\TG = (\G_1,t_1), \dots, (\G_n,t_n)$ and $\TG' = (\G'_1,t'_1), \dots, (\G'_n,t'_n)$,
with 
$\G_i=(V_i,E_i,c_i)$
and $\G_i'=(V_i',E_i',c_i')$.
Moreover, we obtain that
$t=t_i$ and $t'=t'_i$ for some $i \in \{ 1, \dots, n\}$.
Furthermore, $f(v) = u$ for some function ${f:V(\TG) \to V(\TG')}$ satisfying requirements in \Cref{def:iso}.
We define $f': \TV(\TG) \to \TV(\TG')$ such that $f'(w,t_j)=(f(w),t'_j)$  for  all 
$(w,t_j) \in \TV(\TG)$.

Let $\Kone{\TG} = (V,E,R,c)$ and 
$\Kone{\TG} = (V',E',R',c')$.
We will show that $f'$ is an isomorphism between knowledge graphs $\Kone{\TG}$ and $\Kone{\TG'}$.

First, we show that $f'$ is a bijection.
To show that $f'$ is injective, we will show that for any $(w,t_j), (s,t_k) \in   \TV(\TG)$ such that $(w,t_j) \neq (s,t_k)$, we have $f'(w,t_j) \neq  f'(s,t_k)$, that is, $(f(w),t_j) \neq  (f(s),t_k)$.
Since $(w,t_j) \neq (s,t_k)$, we have $w \neq s$ or $t_j \neq t_k$.
If $t_j \neq t_k$, then clearly $(f(w),t_j) \neq  (f(s),t_k)$.
Next assume that $w \neq s$.
By \Cref{def:iso}, function $f$ is injective, thus $f(w)\neq f(s)$, and so, 
$(f(w),t_j) \neq  (f(s),t_k)$.
To show that $f'$ is surjective, let $(s,t_j) \in \TV(\TG')$.
Since $f$ is surjective, there exists $w \in V(\TG)$ such that $f(w)=s$.
Hence, $f'(w,t_j)=(s,t_k)$. Thus $f'$ is indeed a bijection.

To show that $f'$ is an isomorphism between $\Kone{\TG}$ and $\Kone{\TG'}$, it remains to show that for all $(w,t_j), (s,t_k) \in V$ and $r \in R$ the following hold:
\begin{itemize}
\item[(i)]~$c(w,t_j) = c'( f' (w,t_j) )$
and 
\item[(ii)]~$(r,(w,t_j),(s,t_k)) \in E$ if and only if 
$(r,f'(w,t_j),f'(s,t_k)) \in E'$.
\end{itemize}
To show Statement (i), we observe that $c(w,t_j) = c_j(w) = c_j'(f(w)) = c'(f(w),t_j' ) = c'( f' (w,t_j) )$, where
the consecutive equalities hold by: 
the definition of $\Kone{\TG}$, 
the fact that $f$ is an isomorphism between $G_j$ and $G_j'$, 
the definition of $\Kone{\TG'}$,
and the definition of $f'$.

To show Statement (ii), we observe that the following are equivalent:
\begin{itemize}
\item $(r,(w,t_j),(s,t_k)) \in E$,
\item $r = t_k - t_j$, $j \leq k$, and $\{w,s \} \in E_j$,
\item $r = t_k' - t_j'$, $j \leq k$, and $\{f(w),f(s) \} \in E_j'$,
\item $(r,(f(w),t_j'), (f(s),t_k')) \in E'$,
\item $(r,f'(w,t_j),f'(s,t_k) \in E'$,
\end{itemize}
where the consecutive equivalences hold by: 
the definition of $\Kone{\TG}$,
the fact that  $\TG$ and $TG'$ are timewise isomorphic,
the definition of $\Kone{\TG'}$,
and the definition of $f'$.

Hence, we have shown that $f'$ is an isomorphism between $\Kone{\TG}$ and $\Kone{\TG'}$.
Moreover, we have $f'(v,t) = (u,t')$.
Therefore, $\rwl^{(\ell)}(v,t)=\rwl^{(\ell)}(u,t')$, for all $\ell \in \N$.
Thus, by \Cref{thm:Kone1}, we obtain that $\h_{v}^{(\ell)}(t) = \h_{u}^{(\ell)}(t')$ for all global  \MPTGNN{}s.

In the case of local \MPTGNN{}s the proof is similar. We use the same $f'$ and show that it is an isomorphism between 
$\Ktwo{\TG}$ and $\Ktwo{\TG'}$.
Let $\Ktwo{\TG} = (V,E_{\mathsf{loc}},R,c)$ and 
$\Ktwo{\TG} = (V',E'_{\mathsf{loc}},R',c')$;
note that $\Ktwo{\TG}$ and  $\Ktwo{\TG'}$ differ from $\Kone{\TG}$ and  $\Kone{\TG'}$, respectively, only in the definition of edges.
Therefore, since we have already shown that $f'$ is a bijection and that Statement (i) holds, it remains to show that
for all $(w,t_j), (s,t_k) \in V$ and $r \in R$ we have
\begin{itemize}
\item $(r,(w,t_j),(s,t_k)) \in E_{\mathsf{loc}}$ if and only if 
$(r,f'(w,t_j),f'(s,t_k)) \in E'_{\mathsf{loc}}$.
\end{itemize}
The above holds true since the
 following are equivalent:
\begin{itemize}
\item $(r,(w,t_j),(s,t_k)) \in E_{\mathsf{loc}}$,
\item $t_j=t_k$ and there is $m \leq j$ such that $r=t_j-t_m$ and $\{ w,s\} \in E_m$, 
\item $t_j'=t_k'$ and there is $m \leq j$ such that $r=t_j'-t_m'$ and $\{ f(w),f(s)\} \in E'_m$, 
\item $(r,(f(w),t_j'), (f(s),t_k')) \in E'_{\mathsf{loc}}$,
\item $(r,f'(w,t_j),f'(s,t_k)) \in E'_{\mathsf{loc}}$,
\end{itemize}
where the consecutive equivalences hold by: 
the definition of $\Ktwo{\TG}$,
the fact that  $\TG$ and $TG'$ are timewise isomorphic,
the definition of $\Ktwo{\TG'}$,
and the definition of $f'$.
Hence, $f'$ is an isomorphism between 
$\Ktwo{\TG}$ and $\Ktwo{\TG'}$.
Since $f'(v,t) = (u,t')$, we have
 $\rwl^{(\ell)}(v,t)=\rwl^{(\ell)}(u,t')$, for all $\ell \in \N$.
Thus, by \Cref{thm:Kone2}, we obtain that $\h_{v}^{(\ell)}(t) = \h_{u}^{(\ell)}(t')$ for all local  \MPTGNN{}s.
\end{proof}

\section{Proof details for  Relative Expressiveness of Temporal Message Passing Mechanisms}

\globmoreloc*
\begin{proof}
Consider $(b,t_4)$ from $\TG$ in
\Cref{fig:snapshot} and $(b,t_4)$ from $\TG'$ in \Cref{fig:aggregated}~(\subref{fig:snap}).
Knowledge graphs $\Ktwo{\TG}$ and $\Ktwo{\TG'}$ are depicted in \Cref{fig:Kloc}.
Since the snapshot $G_4$ in $\Ktwo{\TG}$ is identical to the snapshot $G_4$ in $\Ktwo{\TG'}$, it is clear that, for any $\ell \in \N$, application of $\rwl^{(\ell)}$ to $\Ktwo{\TG}$ and $\Ktwo{\TG'}$ assigns the same label to $(b,t_4)$ in $\Ktwo{\TG}$ and to $(b,t_4)$ in $\Ktwo{\TG'}$.
Hence, by \Cref{thm:Ktwo1}, local \MPTGNN{}s cannot distinguish these timestamped nodes.

\begin{figure}[ht]
\centering
\begin{tikzpicture}[
dot/.style = {draw, circle, minimum size=#1,
              inner sep=0pt, outer sep=0pt},
dot/.default = 6pt
]
\scriptsize
\pgfmathsetmacro{\tline}{-0.25}
\pgfmathsetmacro{\w}{1.3}
\pgfmathsetmacro{\h}{1.5}
\pgfmathsetmacro{\inh}{0.7}
\pgfmathsetmacro{\dist}{1.9}
\pgfmathsetmacro{\Ax}{0.5}
\pgfmathsetmacro{\Ay}{1.4}
\pgfmathsetmacro{\Bx}{0.4}
\pgfmathsetmacro{\By}{0.5}
\pgfmathsetmacro{\Cx}{1}
\pgfmathsetmacro{\Cy}{1}

\draw[->] (1.8,\tline) -- (9.1,\tline);

\foreach \x in {1,...,4}
{
\draw[fill=gray!90!black,opacity=0.2] (\x*\dist,0) -- (\x*\dist+\w,\inh) -- (\x*\dist+\w,\inh+\h) -- (\x*\dist,\h) -- cycle;

\node at (\x*\dist + 0.87 * \w, \inh+ \h - 0.3) {$G_{\x}$};
\node at (\x*\dist + 0.6 * \w, 0) {$t_{\x}={\x}$};
\draw[-] (\x*\dist + 0.6 * \w, \tline+0.05) -- (\x*\dist + 0.6 * \w, \tline-0.05);

\node[dot=9pt,draw=none] (A\x) at (\x*\dist+\Ax, 0+\Ay) {};
\node[dot=9pt,draw=none] (B\x) at (\x*\dist+\Bx, 0+\By) {};
\node[dot=9pt,draw=none] (C\x) at (\x*\dist+\Cx, 0+\Cy) {};
}

\draw[<->,blue,thick] (A2) -- (B2) node[midway,left] {0};
\draw[<->,blue,thick] (A3) -- (B3) node[midway,left] {1};
\draw[<->,blue,thick] (B3) -- (C3) node[pos=0.3,above] {0};
\draw[<->,blue,thick] (A4) -- (B4) node[midway,left] {2};
\draw[<->,blue,thick] (A4) -- (C4) node[pos=-0.08,right=0.07] {0};
\draw[<->,blue,thick] (B4) -- (C4) node[pos=0.3,above] {1};
\draw (B4) edge[<->,blue,thick,bend right=20] node[pos=0.7,below] {0} (C4);

\node[dot=9pt,draw=black,fill=myblue] at (A1) {$a$};
\node[dot=9pt,draw=black,fill=mygreen] at (B1) {$b$};
\node[dot=9pt,draw=black,fill=myred] at (C1) {$c$};
\node[dot=9pt,draw=black,fill=mygreen] at (A2) {$a$};
\node[dot=9pt,draw=black,fill=mygreen] at (B2) {$b$};
\node[dot=9pt,draw=black,fill=myred] at (C2) {$c$};
\node[dot=9pt,draw=black,fill=mygreen] at (A3) {$a$};
\node[dot=9pt,draw=black,fill=mygreen] at (B3) {$b$};
\node[dot=9pt,draw=black,fill=mygreen] at (C3) {$c$};
\node[dot=9pt,draw=black,fill=myblue] at (A4) {$a$};
\node[dot=9pt,draw=black,fill=mygreen] at (B4) {$b$};
\node[dot=9pt,draw=black,fill=mygreen] at (C4) {$c$};
\end{tikzpicture}
\qquad
\qquad
\begin{tikzpicture}[
dot/.style = {draw, circle, minimum size=#1,
              inner sep=0pt, outer sep=0pt},
dot/.default = 6pt
]
\scriptsize
\pgfmathsetmacro{\tline}{-0.25}
\pgfmathsetmacro{\w}{1.3}
\pgfmathsetmacro{\h}{1.5}
\pgfmathsetmacro{\inh}{0.7}
\pgfmathsetmacro{\dist}{1.9}
\pgfmathsetmacro{\Ax}{0.5}
\pgfmathsetmacro{\Ay}{1.4}
\pgfmathsetmacro{\Bx}{0.4}
\pgfmathsetmacro{\By}{0.5}
\pgfmathsetmacro{\Cx}{1}
\pgfmathsetmacro{\Cy}{1}

\draw[->] (1.8,\tline) -- (9.1,\tline);

\foreach \x in {1,...,4}
{
\draw[fill=gray!90!black,opacity=0.2] (\x*\dist,0) -- (\x*\dist+\w,\inh) -- (\x*\dist+\w,\inh+\h) -- (\x*\dist,\h) -- cycle;

\node at (\x*\dist + 0.87 * \w, \inh+ \h - 0.3) {$G_{\x}$};
\node at (\x*\dist + 0.6 * \w, 0) {$t_{\x}={\x}$};
\draw[-] (\x*\dist + 0.6 * \w, \tline+0.05) -- (\x*\dist + 0.6 * \w, \tline-0.05);

\node[dot=9pt,draw=none] (A\x) at (\x*\dist+\Ax, 0+\Ay) {};
\node[dot=9pt,draw=none] (B\x) at (\x*\dist+\Bx, 0+\By) {};
\node[dot=9pt,draw=none] (C\x) at (\x*\dist+\Cx, 0+\Cy) {};
}

\draw[<->,blue,thick] (A2) -- (B2) node[midway,left] {0};
\draw[<->,blue,thick] (A3) -- (B3) node[midway,left] {1};
\draw[<->,blue,thick] (B3) -- (C3) node[pos=0.3,above] {0};
\draw[<->,blue,thick] (A4) -- (B4) node[midway,left] {2};
\draw[<->,blue,thick] (A4) -- (C4) node[pos=-0.08,right=0.07] {0};
\draw[<->,blue,thick] (B4) -- (C4) node[pos=0.3,above] {1};
\draw (B4) edge[<->,blue,thick,bend right=20] node[pos=0.7,below] {0} (C4);

\node[dot=9pt,draw=black,fill=myblue] at (A1) {$a$};
\node[dot=9pt,draw=black,fill=mygreen] at (B1) {$b$};
\node[dot=9pt,draw=black,fill=mygreen] at (C1) {$c$};
\node[dot=9pt,draw=black,fill=myblue] at (A2) {$a$};
\node[dot=9pt,draw=black,fill=mygreen] at (B2) {$b$};
\node[dot=9pt,draw=black,fill=mygreen] at (C2) {$c$};
\node[dot=9pt,draw=black,fill=myblue] at (A3) {$a$};
\node[dot=9pt,draw=black,fill=mygreen] at (B3) {$b$};
\node[dot=9pt,draw=black,fill=mygreen] at (C3) {$c$};
\node[dot=9pt,draw=black,fill=myblue] at (A4) {$a$};
\node[dot=9pt,draw=black,fill=mygreen] at (B4) {$b$};
\node[dot=9pt,draw=black,fill=mygreen] at (C4) {$c$};
\end{tikzpicture}
\caption{Knowledge graphs $\Ktwo{\TG}$ and  for $\Ktwo{\TG'}$}
\label{fig:Kloc}
\end{figure}

Next, consider $\Kone{\TG}$ and $\Kone{\TG'}$, as depicted in \Cref{fig:glob}.
After application of one iteration of $\rwl$ we obtain that the label $\rwl^{(1)}(b,t_4)$
 in $\Kone{\TG}$ is different from  $\rwl^{(1)}(b,t_4)$ in $\Kone{\TG'}$.
Indeed, $(b,t_4)$ has no blue temporal neighbours in $\Kone{\TG}$, but it has a blue neighbour, namely $(a,t_2)$, in $\Kone{\TG'}$.
Hence, for any $\ell \geq 1$,  application of $\rwl^{(\ell)}$ to $\Kone{\TG}$ and $\Kone{\TG'}$ assigns different labels to $(b,t_4)$ in $\Kone{\TG}$ and to $(b,t_4)$ in $\Kone{\TG'}$.
Therefore, by \Cref{thm:Kone2}, global \MPTGNN{}s can distinguish these nodes.

\begin{figure}[ht]
\centering
\begin{tikzpicture}[
dot/.style = {draw, circle, minimum size=#1,
              inner sep=0pt, outer sep=0pt},
dot/.default = 6pt
]
\scriptsize
\pgfmathsetmacro{\tline}{-0.25}
\pgfmathsetmacro{\w}{1.3}
\pgfmathsetmacro{\h}{1.5}
\pgfmathsetmacro{\inh}{0.7}
\pgfmathsetmacro{\dist}{1.9}
\pgfmathsetmacro{\Ax}{0.5}
\pgfmathsetmacro{\Ay}{1.4}
\pgfmathsetmacro{\Bx}{0.4}
\pgfmathsetmacro{\By}{0.5}
\pgfmathsetmacro{\Cx}{1}
\pgfmathsetmacro{\Cy}{1}

\draw[->] (1.8,\tline) -- (9.1,\tline);

\foreach \x in {1,...,4}
{
\draw[fill=gray!90!black,opacity=0.2] (\x*\dist,0) -- (\x*\dist+\w,\inh) -- (\x*\dist+\w,\inh+\h) -- (\x*\dist,\h) -- cycle;

\node at (\x*\dist + 0.87 * \w, \inh+ \h - 0.3) {$G_{\x}$};
\node at (\x*\dist + 0.6 * \w, 0) {$t_{\x}={\x}$};
\draw[-] (\x*\dist + 0.6 * \w, \tline+0.05) -- (\x*\dist + 0.6 * \w, \tline-0.05);

\node[dot=9pt,draw=none] (A\x) at (\x*\dist+\Ax, 0+\Ay) {};
\node[dot=9pt,draw=none] (B\x) at (\x*\dist+\Bx, 0+\By) {};
\node[dot=9pt,draw=none] (C\x) at (\x*\dist+\Cx, 0+\Cy) {};
}

\draw[<->,blue,thick] (A2) -- (B2) node[pos=0.5,left] {0};
\draw[<-,blue,thick] (A3) -- (B2) node[pos=0.15,above] {1};
\draw[<-,blue,thick] (B3) -- (A2) node[pos=0.12,below] {1};

\draw[<->,blue,thick] (B3) -- (C3) node[pos=0.7,left=0.05] {0};
\draw[<-,blue,thick] (A4) -- (B2) node[pos=0.17,above] {2};
\draw[<-,blue,thick] (B4) -- (A2) node[pos=0.1,below] {2};
\draw[blue,thick] (B4) -- (C3) node[pos=0.7,above] {1};
\draw[blue,thick] (C4) -- (B3) node[pos=0.3,above] {1};
\draw[<->,blue,thick] (A4) -- (C4) node[pos=0.8,above] {0};
\draw[<->,blue,thick] (B4) -- (C4) node[pos=0.7,below] {0};

\node[dot=9pt,draw=black,fill=myblue] at (A1) {$a$};
\node[dot=9pt,draw=black,fill=mygreen] at (B1) {$b$};
\node[dot=9pt,draw=black,fill=myred] at (C1) {$c$};
\node[dot=9pt,draw=black,fill=mygreen] at (A2) {$a$};
\node[dot=9pt,draw=black,fill=mygreen] at (B2) {$b$};
\node[dot=9pt,draw=black,fill=myred] at (C2) {$c$};
\node[dot=9pt,draw=black,fill=mygreen] at (A3) {$a$};
\node[dot=9pt,draw=black,fill=mygreen] at (B3) {$b$};
\node[dot=9pt,draw=black,fill=mygreen] at (C3) {$c$};
\node[dot=9pt,draw=black,fill=myblue] at (A4) {$a$};
\node[dot=9pt,draw=black,fill=mygreen] at (B4) {$b$};
\node[dot=9pt,draw=black,fill=mygreen] at (C4) {$c$};
\end{tikzpicture}
\qquad
\qquad 
\begin{tikzpicture}[
dot/.style = {draw, circle, minimum size=#1,
              inner sep=0pt, outer sep=0pt},
dot/.default = 6pt
]
\scriptsize
\pgfmathsetmacro{\tline}{-0.25}
\pgfmathsetmacro{\w}{1.3}
\pgfmathsetmacro{\h}{1.5}
\pgfmathsetmacro{\inh}{0.7}
\pgfmathsetmacro{\dist}{1.9}
\pgfmathsetmacro{\Ax}{0.5}
\pgfmathsetmacro{\Ay}{1.4}
\pgfmathsetmacro{\Bx}{0.4}
\pgfmathsetmacro{\By}{0.5}
\pgfmathsetmacro{\Cx}{1}
\pgfmathsetmacro{\Cy}{1}

\draw[->] (1.8,\tline) -- (9.1,\tline);

\foreach \x in {1,...,4}
{
\draw[fill=gray!90!black,opacity=0.2] (\x*\dist,0) -- (\x*\dist+\w,\inh) -- (\x*\dist+\w,\inh+\h) -- (\x*\dist,\h) -- cycle;

\node at (\x*\dist + 0.87 * \w, \inh+ \h - 0.3) {$G_{\x}$};
\node at (\x*\dist + 0.6 * \w, 0) {$t_{\x}={\x}$};
\draw[-] (\x*\dist + 0.6 * \w, \tline+0.05) -- (\x*\dist + 0.6 * \w, \tline-0.05);

\node[dot=9pt,draw=none] (A\x) at (\x*\dist+\Ax, 0+\Ay) {};
\node[dot=9pt,draw=none] (B\x) at (\x*\dist+\Bx, 0+\By) {};
\node[dot=9pt,draw=none] (C\x) at (\x*\dist+\Cx, 0+\Cy) {};
}

\draw[<->,blue,thick] (A2) -- (B2) node[pos=0.5,left] {0};
\draw[<-,blue,thick] (A3) -- (B2) node[pos=0.15,above] {1};
\draw[<-,blue,thick] (B3) -- (A2) node[pos=0.12,below] {1};

\draw[<->,blue,thick] (B3) -- (C3) node[pos=0.7,left=0.05] {0};
\draw[<-,blue,thick] (A4) -- (B2) node[pos=0.17,above] {2};
\draw[<-,blue,thick] (B4) -- (A2) node[pos=0.1,below] {2};
\draw[blue,thick] (B4) -- (C3) node[pos=0.7,above] {1};
\draw[blue,thick] (C4) -- (B3) node[pos=0.3,above] {1};
\draw[<->,blue,thick] (A4) -- (C4) node[pos=0.8,above] {0};
\draw[<->,blue,thick] (B4) -- (C4) node[pos=0.7,below] {0};

\node[dot=9pt,draw=black,fill=myblue] at (A1) {$a$};
\node[dot=9pt,draw=black,fill=mygreen] at (B1) {$b$};
\node[dot=9pt,draw=black,fill=mygreen] at (C1) {$c$};
\node[dot=9pt,draw=black,fill=myblue] at (A2) {$a$};
\node[dot=9pt,draw=black,fill=mygreen] at (B2) {$b$};
\node[dot=9pt,draw=black,fill=mygreen] at (C2) {$c$};
\node[dot=9pt,draw=black,fill=myblue] at (A3) {$a$};
\node[dot=9pt,draw=black,fill=mygreen] at (B3) {$b$};
\node[dot=9pt,draw=black,fill=mygreen] at (C3) {$c$};
\node[dot=9pt,draw=black,fill=myblue] at (A4) {$a$};
\node[dot=9pt,draw=black,fill=mygreen] at (B4) {$b$};
\node[dot=9pt,draw=black,fill=mygreen] at (C4) {$c$};
\end{tikzpicture}
\caption{Knowledge graphs $\Kone{\TG}$ and $\Kone{\TG'}$}
\label{fig:glob}
\end{figure}
\end{proof}

\locmoreglob*
\begin{proof}
Consider $(a,t_2)$ from $\TG$ and $(a',t_2)$ from $\TG'$, as depicted in \Cref{fig:Ttwomore}.
Knowledge graphs $\Kone{\TG}$ and  $\Kone{\TG'}$ are depicted in 
\Cref{fig:TtwomoreKone}.
We observe that $(a,t_2)$ in $\Kone{\TG}$ is  isomorphic to $(a',t_2)$ in $\Kone{\TG'}$ so, by \Cref{thm:Kone1}, $(a,t_2)$ and $(a',t_2)$ cannot be distinguished by global \MPTGNN{}s.

\begin{figure}[ht]
\centering
\begin{tikzpicture}[
dot/.style = {draw, circle, minimum size=#1,
              inner sep=0pt, outer sep=0pt},
dot/.default = 6pt
]
\scriptsize
\pgfmathsetmacro{\tline}{-0.25}
\pgfmathsetmacro{\w}{1.3}
\pgfmathsetmacro{\h}{1.5}
\pgfmathsetmacro{\inh}{0.7}
\pgfmathsetmacro{\dist}{1.7}
\pgfmathsetmacro{\Ax}{0.5}
\pgfmathsetmacro{\Ay}{1.4}
\pgfmathsetmacro{\Bx}{0.4}
\pgfmathsetmacro{\By}{0.5}
\pgfmathsetmacro{\Cx}{1}
\pgfmathsetmacro{\Cy}{1}

\draw[->] (1.8,\tline) -- (4.8,\tline);

\foreach \x in {1,...,2}
{
\draw[fill=gray!90!black,opacity=0.2] (\x*\dist,0) -- (\x*\dist+\w,\inh) -- (\x*\dist+\w,\inh+\h) -- (\x*\dist,\h) -- cycle;

\node at (\x*\dist + 0.87 * \w, \inh+ \h - 0.3) {$G_{\x}$};
\node at (\x*\dist + 0.6 * \w, 0) {$t_{\x}$};
\draw[-] (\x*\dist + 0.6 * \w, \tline+0.05) -- (\x*\dist + 0.6 * \w, \tline-0.05);

\node[dot=9pt,draw=none] (A\x) at (\x*\dist+\Ax, 0+\Ay) {};
\node[dot=9pt,draw=none] (B\x) at (\x*\dist+\Bx, 0+\By) {};
\node[dot=9pt,draw=none] (C\x) at (\x*\dist+\Cx, 0+\Cy) {};
}

\draw[<->,blue,thick] (A1) -- (B1) node[pos=0.5,left] {0};
\draw[<-,blue,thick] (B2) -- (A1) node[pos=0.3,below] {1};
\draw[<-,blue,thick] (A2) -- (B1) node[pos=0.35,above] {1};
\draw[<->,blue,thick] (B2) -- (C2) node[pos=0.7,below] {0};

\node[dot=11pt,draw=black,fill=mygreen] at (A1) {$a$};
\node[dot=11pt,draw=black,fill=mygreen] at (B1) {$b$};
\node[dot=11pt,draw=black,fill=mygreen] at (C1) {$c$};
\node[dot=11pt,draw=black,fill=mygreen] at (A2) {$a$};
\node[dot=11pt,draw=black,fill=mygreen] at (B2) {$b$};
\node[dot=11pt,draw=black,fill=mygreen] at (C2) {$c$};
\end{tikzpicture}
\qquad
\qquad
\begin{tikzpicture}[
dot/.style = {draw, circle, minimum size=#1,
              inner sep=0pt, outer sep=0pt},
dot/.default = 6pt
]
\scriptsize
\pgfmathsetmacro{\tline}{-0.25}
\pgfmathsetmacro{\w}{1.3}
\pgfmathsetmacro{\h}{1.5}
\pgfmathsetmacro{\inh}{0.7}
\pgfmathsetmacro{\dist}{1.7}
\pgfmathsetmacro{\Ax}{0.5}
\pgfmathsetmacro{\Ay}{1.4}
\pgfmathsetmacro{\Bx}{0.4}
\pgfmathsetmacro{\By}{0.5}
\pgfmathsetmacro{\Cx}{1}
\pgfmathsetmacro{\Cy}{1}

\draw[->] (1.8,\tline) -- (4.8,\tline);

\foreach \x in {1,...,2}
{
\draw[fill=gray!90!black,opacity=0.2] (\x*\dist,0) -- (\x*\dist+\w,\inh) -- (\x*\dist+\w,\inh+\h) -- (\x*\dist,\h) -- cycle;

\node at (\x*\dist + 0.87 * \w, \inh+ \h - 0.3) {$G'_{\x}$};
\node at (\x*\dist + 0.6 * \w, 0) {$t_{\x}$};
\draw[-] (\x*\dist + 0.6 * \w, \tline+0.05) -- (\x*\dist + 0.6 * \w, \tline-0.05);

\node[dot=9pt,draw=none] (A\x) at (\x*\dist+\Ax, 0+\Ay) {};
\node[dot=9pt,draw=none] (B\x) at (\x*\dist+\Bx, 0+\By) {};
\node[dot=9pt,draw=none] (C\x) at (\x*\dist+\Cx, 0+\Cy) {};
}

\draw[<->,blue,thick] (A1) -- (B1) node[pos=0.5,left] {0};
\draw[<-,blue,thick] (B2) -- (A1) node[pos=0.3,below] {1};
\draw[<-,blue,thick] (A2) -- (B1) node[pos=0.35,above] {1};

\node[dot=11pt,draw=black,fill=mygreen] at (A1) {$a'$};
\node[dot=11pt,draw=black,fill=mygreen] at (B1) {$b'$};
\node[dot=11pt,draw=black,fill=mygreen] at (C1) {$c'$};
\node[dot=11pt,draw=black,fill=mygreen] at (A2) {$a'$};
\node[dot=11pt,draw=black,fill=mygreen] at (B2) {$b'$};
\node[dot=11pt,draw=black,fill=mygreen] at (C2) {$c'$};
\end{tikzpicture}
\caption{Knowledge graphs $\Kone{\TG}$ and  $\Kone{\TG'}$
}
\label{fig:TtwomoreKone}
\end{figure}

Now, consider knowledge graphs $\Ktwo{\TG}$ and  $\Ktwo{\TG'}$ as depicted in 
\Cref{fig:TtwomoreKtwo}.
Note that $(a,t_2)$ in $\Ktwo{\TG}$ has an outgoing path of length 2, whereas $(a,t_2)$ in $\Ktwo{\TG'}$ has only a path of length 1.
Hence, two iterations of $\rwl$ allow us to distinguish these nodes.
Thus, by \Cref{thm:Ktwo2}, $(a,t_2)$ and $(a',t_2)$ can be distinguished by local \MPTGNN{}s.

\begin{figure}[ht]
\centering
\begin{tikzpicture}[
dot/.style = {draw, circle, minimum size=#1,
              inner sep=0pt, outer sep=0pt},
dot/.default = 6pt
]
\scriptsize
\pgfmathsetmacro{\tline}{-0.25}
\pgfmathsetmacro{\w}{1.3}
\pgfmathsetmacro{\h}{1.5}
\pgfmathsetmacro{\inh}{0.7}
\pgfmathsetmacro{\dist}{1.7}
\pgfmathsetmacro{\Ax}{0.5}
\pgfmathsetmacro{\Ay}{1.4}
\pgfmathsetmacro{\Bx}{0.4}
\pgfmathsetmacro{\By}{0.5}
\pgfmathsetmacro{\Cx}{1}
\pgfmathsetmacro{\Cy}{1}

\draw[->] (1.8,\tline) -- (4.8,\tline);

\foreach \x in {1,...,2}
{
\draw[fill=gray!90!black,opacity=0.2] (\x*\dist,0) -- (\x*\dist+\w,\inh) -- (\x*\dist+\w,\inh+\h) -- (\x*\dist,\h) -- cycle;

\node at (\x*\dist + 0.87 * \w, \inh+ \h - 0.3) {$G_{\x}$};
\node at (\x*\dist + 0.6 * \w, 0) {$t_{\x}$};
\draw[-] (\x*\dist + 0.6 * \w, \tline+0.05) -- (\x*\dist + 0.6 * \w, \tline-0.05);

\node[dot=9pt,draw=none] (A\x) at (\x*\dist+\Ax, 0+\Ay) {};
\node[dot=9pt,draw=none] (B\x) at (\x*\dist+\Bx, 0+\By) {};
\node[dot=9pt,draw=none] (C\x) at (\x*\dist+\Cx, 0+\Cy) {};
}

\draw[<->,blue,thick] (A1) -- (B1) node[pos=0.5,left] {0};
\draw[<->,blue,thick] (A2) -- (B2) node[pos=0.5,left] {1};
\draw[<->,blue,thick] (B2) -- (C2) node[pos=0.7,below] {0};

\node[dot=11pt,draw=black,fill=mygreen] at (A1) {$a$};
\node[dot=11pt,draw=black,fill=mygreen] at (B1) {$b$};
\node[dot=11pt,draw=black,fill=mygreen] at (C1) {$c$};
\node[dot=11pt,draw=black,fill=mygreen] at (A2) {$a$};
\node[dot=11pt,draw=black,fill=mygreen] at (B2) {$b$};
\node[dot=11pt,draw=black,fill=mygreen] at (C2) {$c$};
\end{tikzpicture}
\qquad
\qquad
\begin{tikzpicture}[
dot/.style = {draw, circle, minimum size=#1,
              inner sep=0pt, outer sep=0pt},
dot/.default = 6pt
]
\scriptsize
\pgfmathsetmacro{\tline}{-0.25}
\pgfmathsetmacro{\w}{1.3}
\pgfmathsetmacro{\h}{1.5}
\pgfmathsetmacro{\inh}{0.7}
\pgfmathsetmacro{\dist}{1.7}
\pgfmathsetmacro{\Ax}{0.5}
\pgfmathsetmacro{\Ay}{1.4}
\pgfmathsetmacro{\Bx}{0.4}
\pgfmathsetmacro{\By}{0.5}
\pgfmathsetmacro{\Cx}{1}
\pgfmathsetmacro{\Cy}{1}

\draw[->] (1.8,\tline) -- (4.8,\tline);

\foreach \x in {1,...,2}
{
\draw[fill=gray!90!black,opacity=0.2] (\x*\dist,0) -- (\x*\dist+\w,\inh) -- (\x*\dist+\w,\inh+\h) -- (\x*\dist,\h) -- cycle;

\node at (\x*\dist + 0.87 * \w, \inh+ \h - 0.3) {$G'_{\x}$};
\node at (\x*\dist + 0.6 * \w, 0) {$t_{\x}$};
\draw[-] (\x*\dist + 0.6 * \w, \tline+0.05) -- (\x*\dist + 0.6 * \w, \tline-0.05);

\node[dot=9pt,draw=none] (A\x) at (\x*\dist+\Ax, 0+\Ay) {};
\node[dot=9pt,draw=none] (B\x) at (\x*\dist+\Bx, 0+\By) {};
\node[dot=9pt,draw=none] (C\x) at (\x*\dist+\Cx, 0+\Cy) {};
}

\draw[<->,blue,thick] (A1) -- (B1) node[pos=0.5,left] {0};
\draw[<->,blue,thick] (A2) -- (B2) node[pos=0.5,left] {1};

\node[dot=11pt,draw=black,fill=mygreen] at (A1) {$a'$};
\node[dot=11pt,draw=black,fill=mygreen] at (B1) {$b'$};
\node[dot=11pt,draw=black,fill=mygreen] at (C1) {$c'$};
\node[dot=11pt,draw=black,fill=mygreen] at (A2) {$a'$};
\node[dot=11pt,draw=black,fill=mygreen] at (B2) {$b'$};
\node[dot=11pt,draw=black,fill=mygreen] at (C2) {$c'$};
\end{tikzpicture}
\caption{Knowledge graphs $\Ktwo{\TG}$ and  $\Ktwo{\TG'}$
}
\label{fig:TtwomoreKtwo}
\end{figure}
\end{proof}

To prove \Cref{thm:more}, we show first the following lemma.

\begin{lemma}\label{lem:plusk}
Consider $\rwl$ applied to  $\Ktwo{\TG}$, for a colour-persistent temporal graph $\TG$.
Then 
$\rwl^{(\ell)}(v,t) \neq \rwl^{(\ell)}(u,t')$ implies that $\rwl^{(\ell)}(v,t+k) \neq \rwl^{(\ell)}(u,t'+k)$,
for any nodes $(v,t)$, $(u,t')$ in $\Ktwo{\TG}$, any $\ell \in \N$, and any $k \geq 0$ such that $t+k$ and $t'+k$ belong to $\T(\TG)$.
\end{lemma}
\begin{proof}
The proof is by induction on $\ell \in \N$.
Basis, for $\ell=0$, holds by the fact that $\TG$ is colour-persistent, and so $\rwl^{(0)}(v,t) = \rwl^{(0)}(v,t+k)$ as well as $\rwl^{(0)}(u,t') = \rwl^{(0)}(u,t'+k)$.
Thus $\rwl^{(0)}(v,t) \neq \rwl^{(0)}(u,t')$ implies that $\rwl^{(0)}(v,t+k) \neq \rwl^{(0)}(u,t'+k)$.

For the inductive step, we assume that the implication holds for $\ell$, and we will show that it holds also for $\ell+1$.
To this end, assume that $\rwl^{(\ell+1)}(v,t) \neq \rwl^{(\ell+1)}(u,t')$. We will prove that $\rwl^{(\ell+1)}(v,t+k) \neq \rwl^{(\ell+1)}(u,t'+k)$.
Assume that  $\Ktwo{\TG}$ is of the form $(V,E,R,c)$.
Hence, by the definition of $\rwl$, we have 
\begin{align*}
\rwl^{(\ell+1)}(v,t) & = 
\tau \Big( \rwl^{(\ell)}(v,t), 
\ldblbrace ( \rwl^{(\ell)}(w,t'') ,r) \mid (w,t'') \in \NH_{r}(v,t), r \in R \rdblbrace  \Big)   = 
\tau \Big( \rwl^{(\ell)}(v,t), 
A   \Big) ,
\\
\rwl^{(\ell+1)}(u,t') & = 
\tau \Big( \rwl^{(\ell)}(u,t'), 
\ldblbrace ( \rwl^{(\ell)}(w,t'') ,r) \mid (w,t'') \in \NH_{r}(u,t'), r \in R \rdblbrace   \Big)
=
\tau \Big( \rwl^{(\ell)}(u,t'), 
B \Big)
,
\end{align*}
where $\tau$ is injective, whereas $A$ and $B$ stand for the multisets as in the equations above.
Hence, $\rwl^{(\ell+1)}(v,t) \neq \rwl^{(\ell+1)}(u,t')$ implies that $\rwl^{(\ell)}(v,t)\neq \rwl^{(\ell)}(u,t')$, or $A \neq B$.
If $\rwl^{(\ell)}(v,t) \neq \rwl^{(\ell)}(u,t')$ then, by the inductive assumption,
$\rwl^{(\ell)}(v,t+k) \neq \rwl^{(\ell)}(u,t'+k)$.
Thus, by the definition of $\rwl$ (in particular, by injectivity of $\tau$), we obtain that
$\rwl^{(\ell+1)}(v,t+k) \neq \rwl^{(\ell+1)}(u,t'+k)$.

It remains to consider the case  $A \neq B$.
Then, there exists some $r \in R$ such that 
$$
\ldblbrace ( \rwl^{(\ell)}(w,t'') ,r) \mid (w,t'') \in \NH_{r}(v,t) \rdblbrace   
\neq 
\ldblbrace ( \rwl^{(\ell)}(w,t'') ,r) \mid (w,t'') \in \NH_{r}(u,t') \rdblbrace  .
$$
We observe that, by the definition of $\Ktwo{\TG}$, we obtain that  $(w,t'') \in \NH_{r}(v,t)$
implies that $t''=t$, and similarly
$(w,t'') \in \NH_{r}(u,t')$ implies that $t''=t$.
Therefore the following holds:
$$
\ldblbrace ( \rwl^{(\ell)}(w,t) ,r) \mid (w,t) \in \NH_{r}(v,t) \rdblbrace   
\neq 
\ldblbrace ( \rwl^{(\ell)}(w,t') ,r) \mid (w,t') \in \NH_{r}(u,t') \rdblbrace  .
$$
Moreover, 
$(w,t) \in \NH_{r}(v,t)$ is equivalent to  $(w,t) \in \NH_{r+k}(v,t+k)$, whereas
$(w,t') \in \NH_{r}(u,t')$ is equivalent to  $(w,t') \in \NH_{r+k}(u,t'+k)$.
Hence,  for $r' = r +k$, we obtain that
$$
\ldblbrace ( \rwl^{(\ell)}(w,t'') ,r') \mid (w,t'') \in \NH_{r'}(v,t+k) \rdblbrace   
\neq 
\ldblbrace ( \rwl^{(\ell)}(w,t'') ,r') \mid (w,t'') \in \NH_{r'}(u,t'+k) \rdblbrace  .
$$
This, by the definition of $\rwl$, we obtain that
$\rwl^{(\ell+1)}(v,t+k) \neq \rwl^{(\ell+1)}(u,t'+k)$.
\end{proof}

\more*
\begin{proof}
By \Cref{thm:locmoreglob}, there are timestamped nodes in colour-persistent graphs which can be distinguished by local, but not by global \MPTGNN{}s.
Hence, to prove the theorem we need so show that  whenever $(v,t)$ and $(u,t')$, from a  colour-persistent temporal graph $\TG$ (without loss of generality we assume that $(v,t)$ and $(u,t')$ belong to the same $\TG$; if this is not the case we can always take a disjoint union of temporal graphs to which these timestamped nodes belong to), can be distinguished by global \MPTGNN{}s, they can also be  distinguished  by local \MPTGNN{}s.
For this,  by \Cref{thm:Kone1} and \Cref{thm:Ktwo2},
it suffices to show the following statement:
\begin{itemize}
\item[($\ast$)]
$\rwlg^{(\ell)}(v,t) \neq \rwlg^{(\ell)}(u,t')$   implies $\rwll^{(\ell)}(v,t) \neq \rwll^{(\ell)}(u,t')$,

for any $\ell \in \N$, and any timestamped nodes $(v,t)$, $(u,t')$,
where $\rwlg^{(\ell)}$ is  application of $\rwl$ to $\Kone{\TG}$, and $\rwll^{(\ell)}$ is  application of $\rwl$ to $\Ktwo{\TG}$.
\end{itemize}
Within the proof we will use $\NHg$ and $\NHl$ for, respectively, 
$\NH_{r}$ in $\Kone{\TG}$ and $\NH_{r}$ in $\Ktwo{\TG}$.

We will show Statement ($\ast$) by induction on $\ell$, using \Cref{lem:plusk}.
To this end, we will use notation $\Kone{\TG}=(V,E,R,c)$ and $\Ktwo{\TG}=(V,E',R,c)$; note that, by the definition, all the components in these knowledge graphs are the same except labelled edges $E$ and $E'$.

\paragraph{Basis.} The basis of the induction, with $\ell=0$, holds trivially, since the sets of nodes $V$ and their initial colourings $c$ are the same in $\Kone{\TG}$ and $\Ktwo{\TG}$.
Hence, $\rwlg^{(0)}(v,t)=\rwll^{(0)}(v,t)$  and $\rwlg^{(0)}(u,t')=\rwll^{(0)}(u,t')$. Thus, Statement ($\ast$) holds for $\ell=0$.

\paragraph{Inductive step.} For the inductive step, we assume that Statement~($\ast$) holds for $\ell$. 
We will show that it holds for $\ell+1$.
To this end, assume that $\rwlg^{(\ell+1)}(v,t) \neq \rwlg^{(\ell+1)}(u,t')$, for some $(v,t)$ and $(u,t')$.
By the definition of $\rwl$, we have
\begin{align*}
\rwlg^{(\ell+1)}(v,t) & = 
\tau \Big( \rwlg^{(\ell)}(v,t), 
\ldblbrace ( \rwlg^{(\ell)}(w,t'') ,r) \mid (w,t'') \in \NHg(v,t), r \in R \rdblbrace   \Big)  ,
\\
\rwlg^{(\ell+1)}(u,t') & = 
\tau \Big( \rwlg^{(\ell)}(u,t'), 
\ldblbrace ( \rwlg^{(\ell)}(w,t'') ,r) \mid (w,t'') \in \NHg(u,t'), r \in R \rdblbrace   \Big).
\end{align*}
Hence, $\rwlg^{(\ell+1)}(v,t) \neq \rwlg^{(\ell+1)}(u,t')$ implies that $\rwlg^{(\ell)}(v,t)\neq \rwlg^{(\ell)}(u,t')$, or the two multisets in the above equations are different.
If $\rwlg^{(\ell)}(v,t) \neq \rwlg^{(\ell)}(u,t')$, then by the inductive assumption we obtain that
$\rwll^{(\ell)}(v,t) \neq \rwll^{(\ell)}(u,t')$.
Thus, by the definition of $\rwll$,  we obtain that
$\rwll^{(\ell+1)}(v,t) \neq \rwll^{(\ell+1)}(u,t')$.

Now,  consider the case  when the two multisets above are different.
Then, there exists some $r \in R$ such that 
the multisets of colours of temporal $r$-neighbourhoods of $(v,t)$ and $(u,t')$ computed by $\rwlg^{(\ell)}$ are different.
We let $A$ and $B$ be these neighbourhoods, namely
\begin{align*}
A & = \{ (w,t'')  \mid (w,t'') \in \NHg(v,t) \}, 
\\
B &= \{ (w,t'')  \mid (w,t'') \in \NHg(u,t') \} ,
\end{align*}
and we let $\Ag$ and $\Bg$ be the multisets consisting of colours of elements of $A$ and $B$, respectively, that is
\begin{align*}
\Ag & = \ldblbrace  \rwlg^{(\ell)}(w,t'')  \mid (w,t'') \in A \rdblbrace, 
\\
\Bg &= \ldblbrace   \rwlg^{(\ell)}(w,t'') \mid (w,t'') \in B \rdblbrace. 
\end{align*}
Hence, we have $\Ag \neq \Bg$. 

Next we define the following multisets $\Al$ and $\Bl$ of colours computed by $\rwll$ (they correspond to $\Ag$ and $\Bg$ computed by $\rwlg$):
\begin{align*}
\Al & = \ldblbrace  \rwll^{(\ell)}(w,t'')  \mid (w,t'') \in A \rdblbrace, 
\\
\Bl &= \ldblbrace   \rwll^{(\ell)}(w,t'') \mid (w,t'') \in B \rdblbrace, 
\end{align*}
and the colours of the same nodes but at time points $t$ and $t'$, namely:
\begin{align*}
\Al^t & = \ldblbrace  \rwll^{(\ell)}(w,t)  \mid (w,t'') \in A \rdblbrace, 
\\
\Bl^{t'} &= \ldblbrace   \rwll^{(\ell)}(w,t') \mid (w,t'') \in B \rdblbrace. 
\end{align*}
We observe that, by the definition of $\Ktwo{\TG}$,
the multiset $\Al^t$ consists of colours assigned by the $\ell$th iteration of $\rwl$ to  nodes $(w,t)$ in $\Ktwo{\TG}$ witch are connected with $(v,t)$ by an edge labelled by $r$.
Similarly, $\Bl^{t'}$ consists of colours of  nodes $(w,t')$ in $\Ktwo{\TG}$ with an edge to $(u,t')$ labelled by $r$.
Hence, to prove that $\rwll^{(\ell+1)}(v,t) \neq \rwll^{(\ell+1)}(u,t')$, it suffices to show that $\Al^t \neq \Bl^{t'}$.
We will do it in two steps.
First, we will show that 
$\Al \neq \Bl$.
Second, we will use this fact to show that 
$\Al^t \neq \Bl^{t'}$.

\paragraph{Step 1.} We will show that $\Al \neq \Bl$.

Suppose towards a contradiction that $\Al = \Bl$.
Hence, there is a bijection $f:A \to B$ such that  $\rwll^{(\ell)}(w,t'') =  \rwll^{(\ell)}(f(w,t''))$, for each $(w,t'') \in A$.
Since $\Ag \neq \Bg$, there is a colour $d$ such that
the following sets have different cardinalities:
\begin{align*}
A_d & = \{ (w,t'') \in A \mid \rwlg^{(\ell)}(w,t'')  =d  \},
\\
B_d & = \{ (w,t'') \in B \mid \rwlg^{(\ell)}(w,t'')  =d  \}.
\end{align*}
Without loss of generality let us assume that $|B_d| < |A_d|$.

Next, consider the image of $A_d$ under $f$, that is, the set  $f(A_d) = \{f(w,t'') \mid (w,t'') \in A_d  \}$.
By the definition of $f$, we have $\rwll^{(\ell)}(w,t'') = \rwll^{(\ell)} ( f(w,t'') )$ for each $(w,t'') \in A_d$.
Moreover, by  the inductive assumption, if
$\rwll^{(\ell)}(w_1,t_1) = \rwll^{(\ell)}(w_2,t_2)$, then 
$\rwlg^{(\ell)}(w_1,t_1) = \rwlg^{(\ell)}(w_2,t_2)$.
Hence, if $(w,t'') \in f(A_d)$, then 
$\rwlg^{(\ell)}(w,t'') = d$, and so,
$(w,t'') \in B_d$.
Consequently, $|f(A_d)| \leq |B_d|$, and thus
$|f(A_d)| \leq |B_d| < |A_d|$.
Therefore, $|f(A_d)|  < |A_d|$, and so, $f$ is not a bijection, which raises  a contradiction and finished the proof of Step 1.

\paragraph{Step 2.} We will show that $\Al^t \neq \Bl^{t'}$.

To this end we will use similar argumentation as in Step 1, but in the place of the inductive assumption we will use \Cref{lem:plusk}.
Let us define two sets:
\begin{align*}
C & = \{ (w,t) \mid (w,t'') \in A \},
\\
D & = \{ (w,t') \mid (w,t'') \in B \} .
\end{align*}
Suppose towards a contradiction that $\Al^t = \Bl^{t'}$.
Hence, there is a bijection $f:C \to D$ such that  $\rwll^{(\ell)}(w,t) =  \rwll^{(\ell)}(f(w,t))$, for each $(w,t) \in C$.
Since $\Al \neq \Bl$, there is a colour $e$ such that
the following sets have different cardinalities:
\begin{align*}
A_e & = \{ (w,t'') \in A \mid \rwll^{(\ell)}(w,t'')  =e  \},
\\
B_e & = \{ (w,t'') \in B \mid \rwll^{(\ell)}(w,t'')  =e  \}.
\end{align*}
Without loss of generality let us assume that $|B_e| < |A_e|$.

For the set $A_e$ we define the corresponding set of timestampted nodes $C_e \subseteq C$ with time point $t$.
For $B_e$ we define the corresponding set of timestampted nodes $D_e \subseteq D$ with time point $t'$, namely:
\begin{align*}
C_e & = \{ (w,t)  \mid (w,t'') \in A_e \},
\\
D_e & =  \{ (w,t')  \mid (w,t'') \in B_e \}.
\end{align*}

Next, consider the image of $C_e$ under $f$, that is, the set  $f(C_e) = \{f(w,t) \mid (w,t) \in C_e  \}$.
By the definition of $f$, we have $\rwll^{(\ell)}(w,t) = \rwll^{(\ell)} ( f(w,t) )$ for each $(w,t) \in C_e$.
Moreover, by \Cref{lem:plusk}, if
$\rwll^{(\ell)}(w_1,t) = \rwll^{(\ell)}(w_2,t')$, then 
$\rwll^{(\ell)}(w_1,t-r) = \rwll^{(\ell)}(w_2,t'-r)$.
Hence, if $(w,t') \in f(C_e)$, then 
$\rwll^{(\ell)}(w,t'-r) = e$, and so,
$(w,t'-r) \in B_e$.
Consequently, $|f(C_e)| \leq |B_e|$, and thus
$|f(C_e)| \leq |B_e| < |A_e| = |C_e|$ (note that $|A_e| = |C_e|$ holds by the definition of $C_e$).
Therefore, $|f(C_e)|  < |C_e|$, and so, $f$ is not a bijection, which raises  a contradiction and finished the proof of Step 2.
\end{proof}

\section{Experiments}

In \Cref{tab:benchmark-statistics} we present statistics for the family (link prediction) of TGB benchmarks we used in experiments.

\begin{table}[ht]
    \caption{Link prediction benchmark statistics, adapted from \url{https://tgb.complexdatalab.com/docs/linkprop/}, where
    ``steps'' indicates the number of distinct time points in a temporal graph, which may include more than one edge, and
    ``surprise'' is the proportion of test edges not seen during training}
    \centering
    \begin{tabular}{l l r r r r r r}
scale & name & nodes & edges & steps & surprise & metric\\
\midrule
small & \texttt{tgbl-wiki} & 9,227 & 157,474 & 152,757 & 0.108 & MRR\\
small & \texttt{tgbl-review} & 352,637 & 4,873,540 & 6,865 & 0.987 & MRR\\
medium & \texttt{tgbl-coin} & 638,486 & 22,809,486 & 1,295,720 & 0.120 & MRR
    \end{tabular}
    \label{tab:benchmark-statistics}
\end{table}

In \Cref{tab:layer-results} we provide  raw data used to produce \Cref{fig:layer-results}. 

\begin{table}[ht]
    \caption{MRR scores against the number $\ell$ of layers  on \texttt{tgbl-wiki}, used to generate \Cref{fig:layer-results}}
    \label{tab:layer-results}
    \centering
    \begin{tabular}{r r r r r r r r r}
        & 1 & 2 & 3 & 4 & 5 & 6 & 7 & 8 \\
        global & 0.044 & 0.061 & 0.199 & 0.223 & 0.294 & 0.225 & 0.133 & 0.126 \\
        local & 0.040 & 0.089 & 0.235 & 0.264 & 0.302 & 0.421 & 0.430 & 0.400 \\
    \end{tabular}
\end{table}

\end{document}